\documentclass[letterpaper]{article} 
\usepackage[]{aaai23}  
\usepackage{times}  
\usepackage{helvet}  
\usepackage{courier}  
\usepackage[hyphens]{url}  
\usepackage{graphicx} 
\usepackage{natbib}  
\usepackage{caption} 
\usepackage{algorithm}
\usepackage{algorithmic}
\usepackage{subcaption}
\usepackage{booktabs}
\usepackage{amsmath}
\usepackage{amssymb}
\usepackage{mathtools} 
\usepackage{amsthm}

\newtheorem{proposition}{Proposition}
\usepackage{booktabs} 
\usepackage{tikz} 
\usepackage{makecell}
\usepackage{amsmath}
\usepackage{amsfonts}
\usepackage{amssymb}
\usepackage[hidelinks]{hyperref}

\usepackage{amsmath}
\DeclareMathOperator*{\argmax}{arg\,max}
\DeclareMathOperator*{\argmin}{arg\,min}
\newcommand\figref[1]{Fig.~\ref{#1}}
\newcommand\equref[1]{Eq.~\eqref{#1}}
\newcommand\secref[1]{Sec.~\ref{#1}}
\usepackage{newfloat}
\usepackage{listings}
\DeclareCaptionStyle{ruled}{labelfont=normalfont,labelsep=colon,strut=off} 
\floatstyle{ruled}
\newfloat{listing}{tb}{lst}{}
\floatname{listing}{Listing}
\title{The Effect of Modeling Human Rationality Level\\ on Learning Rewards from Multiple Feedback Types}
\author{
     Gaurav R. Ghosal\textsuperscript{\rm 1}\equalcontrib,
     Matthew Zurek\textsuperscript{\rm 2}\equalcontrib,
     Daniel S. Brown\textsuperscript{\rm 3},
     Anca D. Dragan
     \textsuperscript{\rm 1} 
}
\affiliations{
    \textsuperscript{\rm 1}University of California, Berkeley, \textsuperscript{\rm 2}University of Wisconsin, Madison, \textsuperscript{\rm 3}University of Utah\\
    \{gauravrghosal,anca\}@berkeley.edu, \, matthewzurek@cs.wisc.edu, \, dsbrown@cs.utah.edu

}

\usepackage{bibentry}

\begin{document}

\maketitle

\begin{abstract}
When inferring reward functions from human behavior (be it demonstrations, comparisons, physical corrections, or e-stops), it has proven useful to model the human as making noisy-rational choices, with a ``rationality coefficient" capturing how much noise or entropy we expect to see in the human behavior. Prior work typically sets the rationality level to a constant value, regardless  of the type, or quality, of human feedback. However, in many settings, giving one type of feedback (e.g. a demonstration) may be much more difficult than a different type of feedback (e.g. answering a comparison query). Thus, we expect to see more or less noise depending on the type of human feedback. In this work, we advocate that grounding the rationality coefficient in real data for each feedback type, rather than assuming a default value, has a significant positive effect on reward learning. We test this in both simulated experiments and in a user study with real human feedback. We find that overestimating human rationality can have dire effects on reward learning accuracy and regret. We also find that fitting the rationality coefficient to human data enables better reward learning, even when the human deviates significantly from the noisy-rational choice model due to systematic biases. Further, we find that the rationality level affects the informativeness of each feedback type: surprisingly, demonstrations are not always the most informative---when the human acts very suboptimally, comparisons actually become more informative, even when the rationality level is the same for both.  Ultimately, our results emphasize the importance and advantage of paying attention to the assumed human-rationality-level, especially when agents actively learn from multiple types of human feedback.

\end{abstract}

\section{Introduction}
\begin{figure*}[hbt!]
    \centering
    \includegraphics[width=0.75\textwidth]{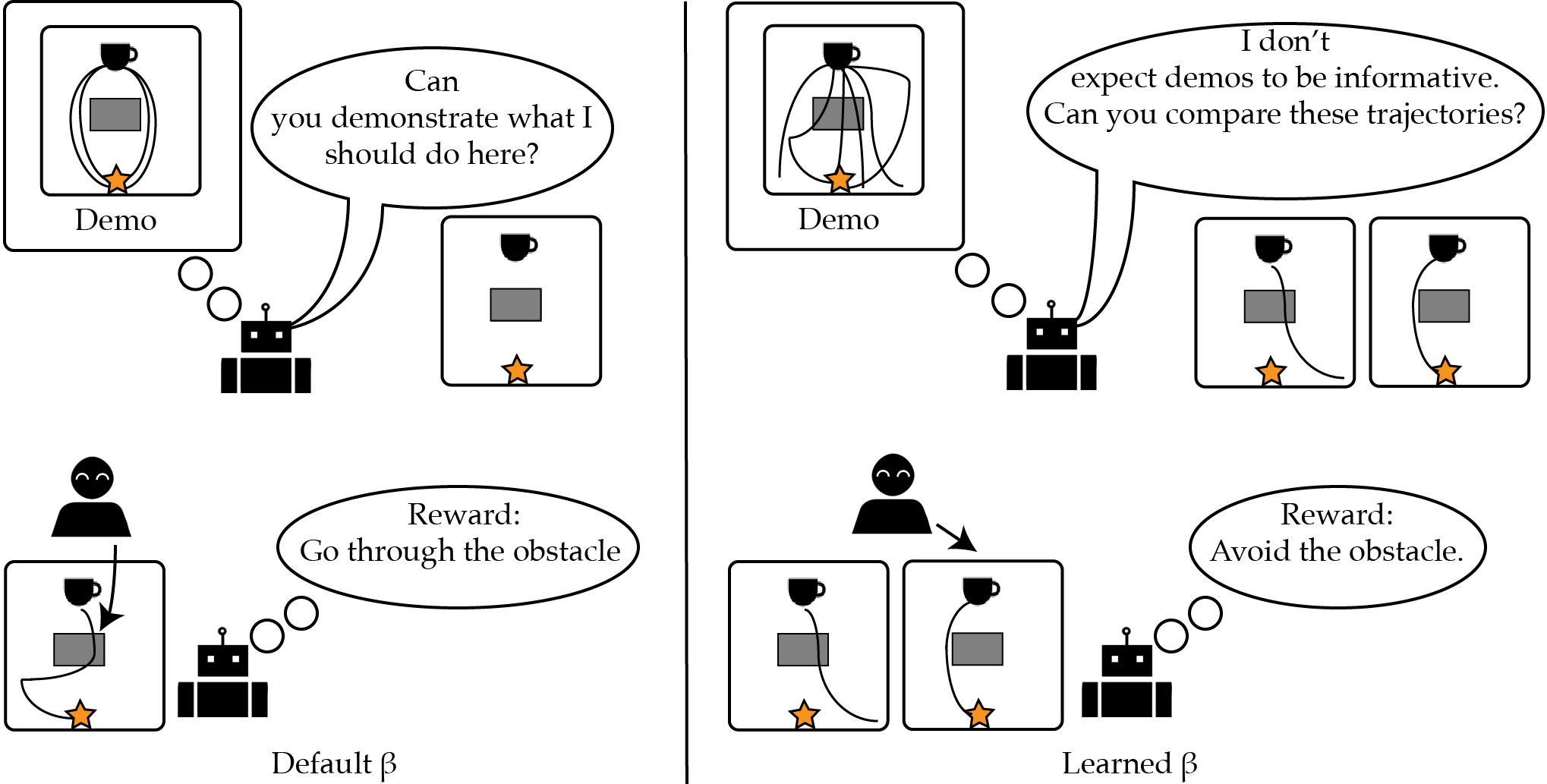}
    \label{figure:frontfig}
    \caption{\textbf{Benefits of Considering Human Rationality}. We depict reward inference via active querying over multiple feedback types. In the Default $\beta$ setting (left panel), the robot erroneously  assumes that the human is rational in giving demonstrations and consequently infers the reward poorly. Under the Learned $\beta$ setting (right panel), the robot anticipates the human's irrationality in demonstrations and is able to query for comparisons instead. As a result, the inferred reward more closely aligns with the human's intention.}
    \label{fig:frontfig}
\end{figure*}
Reward learning started from the inverse optimal control idea that we can recover the underlying objective when observing optimal behavior~\cite{kalman1964linear}, and transitioned into AI with the introduction of inverse reinforcement learning~\cite{ng2000algorithms}. While initial research assumed optimal demonstrators~\cite{ng2000algorithms,ratliff2006maximum}, the field quickly moved to the noisy-rational human model~\cite{morgenstern1953theory}: a number of simultaneous works, with different motivations, converged on a Bolzmann (maximum entropy) distribution, where the human actions are exponentially more probable the higher value they are~\cite{baker2007goal,ramachandran2007bayesian,ziebart2008maximum,henry2010learning, vasquez2014inverse,kretzschmar2016socially,kitani2012activity,wulfmeier2015maximum,brown2018efficient,christiano2017deep,finn2016guided,mainprice2015predicting}.
Often this model would have a "rationality" coefficient $\beta$\footnote{Sometimes denoted by $\alpha$ and sometimes as the inverse ($1/\beta$).} meant to capture how good of an optimizer the human is---setting $\beta$ to 0 would yield the uniform distribution capturing a random human, while $\beta\rightarrow\infty$ would put all the probability mass on optimal actions.

Inspired by the way economists look at preferences, the field then started looking beyond learning from demonstrations to learning from comparisons~\cite{wirth2017survey,christiano2017deep,biyik2020learning}. The model was similar: still a Boltzmann distribution, but over two trajectories/actions, instead of over all possible trajectories/actions. Other researchers started looking at a deluge of feedback types: comparisons~\cite{wirth2017survey}, language~\cite{matuszek2012joint}, demonstrations~\cite{ng2000algorithms}, trajectory rankings \cite{brown2020safe}, corrections~\cite{bajcsy2017learning}, critiques~\cite{cui2018active}, e-stops~\cite{hadfield2017off}, binary feedback~\cite{knox2009interactively}, and proxy rewards~\cite{hadfield2017inverse}. Recently, it was shown that all of these can be interpreted as noisy-rational (Boltzmann) choices~\citep{jeon2020reward}, opening the door to learning from all of these feedback types in combination, and even enabling robots to actively select feedback types.

Boltzmann-rationality's ability to unify different feedback types is useful, but the model comes with this one parameter, $\beta$, which begs the question: what should we set that to? Prior work often either omits $\beta$ (implicitly setting it to 1)~\cite{finn2016guided,christiano2017deep,ibarz2018reward} or sets it to a fixed, often heuristic, value across all feedback types~\cite{ramachandran2007bayesian,shah2019feasibility,biyik2020learning,jeon2020reward}. But demonstrations are sometimes easier or harder to give, depending on the task and the interface, suggesting that $\beta$ should be adapted to the domain. And comparisons might be much easier to answer than demonstrations, suggesting we should be using a higher $\beta$ for the former. Our goal in this work is to answer the question: does this matter? Are there real benefits to grounding $\beta$ in real data for each feedback type, or is it safe to stick to a default value?

We analyze this in both simulation and a user study. Our investigation begins in the single-feedback regime, where we isolate the key role played by $\beta$ in \emph{interpreting feedback}. Through empirical and theoretical analysis, we conclusively demonstrate that overestimating $\beta$ is particularly harmful for reward inference and that inferring $\beta$ benefits reward learning by avoiding this. Importantly, our results generalize across varying forms of biased and noisy behavior, beyond Boltzmann rationality. We find that many known systematic biases can be approximated by a noisy-rational model with a learned $\beta$, enabling better reward inference. We demonstrate this in simulation with specified biases as well as in a user study, where feedback contains arbitrary human biases.

We next study the importance of $\beta$ when an agent actively learns from multiple feedback types. In this setting, we observe a new role played by $\beta$: in addition to controlling the reward inference, the estimate of $\beta$ also affects the selection of which type of feedback should be queried. In particular, we see that $\beta$ strongly affects the informativeness of different feedback types. Surprisingly, this is true even when there is a shared rationality level across feedback types: at low $\beta$ values, we show that comparisons are more informative than demonstrations, while demonstrations gain an advantage at higher $\beta$ values. As a result of $\beta$'s role in both query selection and feedback interpretation, setting it correctly has a significant impact on performance and different settings of beta significantly change the queries selected by active learning. Notably, our findings show the insufficiency of relying on popular heuristics such as starting with demonstrations and fine-tuning with comparisons \cite{ibarz2018reward,palan2019learning,biyik2020learning,liu2023efficient}. Rather, we demonstrate that accounting for rationality is essential for active learning to uncover the feedback query that is truly most informative.

Overall, we contribute an analysis of the effects of estimating a human's rationality level on the quality of reward learning and demonstrate the importance of using the $\beta$ parameter in a principled way over heuristic approaches. In particular, we show that setting $\beta$ appropriately  becomes increasingly crucial as we develop agents that actively learn from multiple types of human feedback. Our analysis can be summarized by the following practical findings:
\begin{enumerate}
    \item Modeling human behavior with Boltzmann rationality provides benefits even in the face of harder to model systematic biases.
    \item When accurate estimates of $\hat{\beta}$ cannot be found, one should err on the side of underestimation.
    \item The success of active learning over multiple feedback types depends strongly on an accurate estimate of $\hat{\beta}$.
    \item The most informative feedback type varies as a function of the human's rationality, even when the feedback types share a common rationality level.
    \item It is possible to obtain good estimates of $\hat{\beta}$ by obtaining a small amount of calibration feedback (where the human optimizes a known reward).
\end{enumerate}

\section{Formulation}
\noindent\textbf{Preliminaries and Notation.}
We model the environment as a finite horizon Markov decision process (MDP) with states $s \in \mathcal{S}$, actions $a \in \mathcal{A}$, and transition dynamics $P(s' \mid s, a)$. The reward function $r: (S \times A) \rightarrow \mathbb{R}$ is initially unknown to the robot but is communicated by a human through multiple forms of feedback, such as demonstrations of desired behavior, preference comparisons between trajectories, and corrective interventions. 

Following prior work \cite{jeon2020reward}, we interpret these varying forms of feedback as a \textit{reward-rational choice} over a (potentially implicit)  choice set $\mathcal{C}$. In this work, we study a robot that \emph{actively chooses} from multiple feedback types. To facilitate this, we model a query for human feedback as a Bayesian experimental design problem ~\cite{chaloner1995bayesian}. We define feedback types as functions from designs to (choice set, grounding) pairs and active learning as the optimization over designs for information gain. For each possible feedback type, the robot has a choice over different possible experimental designs, $\mathcal{X}$, where the experiment design must be specified before the human can provide data. Given an experiment design $x \in \mathcal{X}$ and a human choice $c \in \mathcal{C}(x)$, we use the grounding function $\varphi(x,c)$ to ground the human feedback into the space of trajectories, $\Xi$. The grounded trajectory is interpreted to be Boltzmann rational under the human's reward function. A trajectory, $\xi$, is defined as a sequence of state-action pairs. We use $\xi_{i:j}$ to denote the sub-trajectory starting with the $i$th state-action pair and ending with the $j$th state-action pair.

\subsection{Human Feedback Types}

In this work, we study the three feedback types below.
We refer the reader to Jeon et al.~\cite{jeon2020reward} for a discussion of how many other feedback types can be formalized similarly.

\vspace{1mm}
\noindent \textbf{Demonstrations} can be viewed as a sequence of explicit choices over actions conditioned on states.\footnote{Jeon et al.~\cite{jeon2020reward} model human demonstrations as choices over all possible trajectories; however, with stochastic dynamics, human actions are conditioned on observed state transitions and the human cannot pre-select a specific trajectory.}

The design is all starting states: $\mathcal{X} = \mathcal{S}$ and the grounding function is identity.
For a demonstration $\xi$ starting from state $s_0$ we have the following observation model.

\begin{align}
    \begin{split}
    P(\xi \mid r, \beta) =& \prod_{(s_t,a_t) \in \xi} \pi_\beta (a_t \mid s_t)\\
    =& \prod_{(s_t,a_t) \in \xi} \frac{\exp(\beta Q_{t}^{\rm soft}(s_t,a_t \mid r))}{\sum_{b \in \mathcal{A}} \exp(\beta Q_{t}^{\rm soft}(s_t,b \mid r))}
    \end{split}
\end{align}
where $Q_{t}^{\rm soft}(s,a \mid r) = r(s,a) + \gamma \mathbb{E}_{s'}[V_{t+1}^{\rm soft}(s')]$, and $V_{t}^{\rm soft}(s) = \mathbb{E}_{a \sim \pi_\beta} [ Q_{t}^{\rm soft}(s,a) - \log \pi_\beta(a \mid s)]$ are the soft Q-function, and Value function, respectively~\cite{kitani2012activity,haarnoja2017reinforcement}, and $\pi_\beta$ is the corresponding (time-dependent) policy.

\vspace{1mm}
\noindent \textbf{Comparisons} are a choice between two trajectories. Thus, the possible designs are $\mathcal{X} = \Xi^2$, all pairs of trajectories and the grounding function maps to the preferred trajectory. The likelihood that the human prefers trajectory $A$ over $B$ is given by the Bradley-Luce-Shepherd rule~\cite{bradley1952rank,christiano2017deep}:
\begin{equation}
    \label{comp_eq}
    P(\xi_A \mid r, \beta) = \frac{\exp \left( \beta \cdot r(\xi_A) \right)} { \exp \left( \beta \cdot r(\xi_A) \right)  + \exp \left( \beta \cdot r(\xi_B) \right)}  
\end{equation}

\noindent \textbf{E-stops} represent the intervention of a human telling the robot to stop rather than continue its trajectory. We assume that the human is able to observe the robot's planned trajectory and then selects a desired stopping point $t$ at which point the episode terminates. Thus, the space of possible designs is $\mathcal{X} = \Xi$, all trajectories. The choice set is the stopping time $t$, and the grounding function is the sub-trajectory $\xi_{0:t}$. Given a robot trajectory $\xi$, we have the following likelihood function for the human's choice $c_h = t$:
\begin{equation}\label{eq:likelihood_rric}
    P(t \mid \xi, r, \beta) = \frac{\exp (\beta \cdot r(\xi_{0:t}))}{\sum_{k = 0}^T \exp (\beta \cdot r(\xi_{0:k})) }.
\end{equation}

\subsection{Estimating a Human's Rationality Level from Data}
Rather that assuming a known or constant value for the rationality coefficient, we study the effect of learning an estimate of the human's rationality level, $\hat{\beta}$, from human data.
As a vehicle for our analysis, we consider access to a separate calibration phase where we present the human with a known, calibration reward function, $r'$, and then ask them to provide feedback (e.g., demonstrations, comparisons, e-stops) with respect to this reward function. The benefit of this calibration is that given human feedback that corresponds to a known reward function, we can find $\hat{\beta}$ that maximizes the log-likelihood of \equref{eq:likelihood_rric}: 
\begin{equation}\label{eq:beta_hat_fit_eq}
    \begin{split}
    \hat{\beta} =& \argmax_{\beta} \beta \cdot \mathbb{E}_{\xi \sim \varphi(x, c_h)}[r(\xi)]\\
    -&\log \sum_{c \in \mathcal{C}} \exp \left( \beta \cdot \mathbb{E}_{\xi \sim \varphi(x, c)}[r(\xi)] \right),
    \end{split}
\end{equation}
since $c_{h}$ is collected from the human during calibration and $r',x,C,\varphi$ are constant and known to the robot. This approach intentionally favors simplicity over real-world practicality---our focus is on assessing the importance of having a good model of the human's rationality level---in practice, one could also fit $\hat{\beta}$ on human data optimizing an unknown $r$ by marginalizing over $r$.
\subsection{Active Learning over Feedback Types}

We consider the scenario in which the robot can actively query the most informative feedback given its current belief over a parameterized reward function. We can cast this as the problem of selecting a design $\mathcal{X}$ which optimizes the expected information gain over the possible human feedback induced by $\mathcal{X}$. Concretely, this can be written as the following optimization problem 
\begin{equation}
    \max_{x \in \mathcal{X}} \mathbb{E}_{c_h \sim P(c_h|x)}\left[D_{KL}\big(P(\theta | c_h, x) \| P(\theta)\big)\right],
\end{equation}
in which we consider $P(\theta)$ to be our prior distribution over the reward function, $r_\theta$, parameterized by $\theta$.

\begin{figure*}[hbt!]
     \centering
     \begin{subfigure}[b]{0.32\textwidth}
         \centering
         \includegraphics[width=\textwidth]{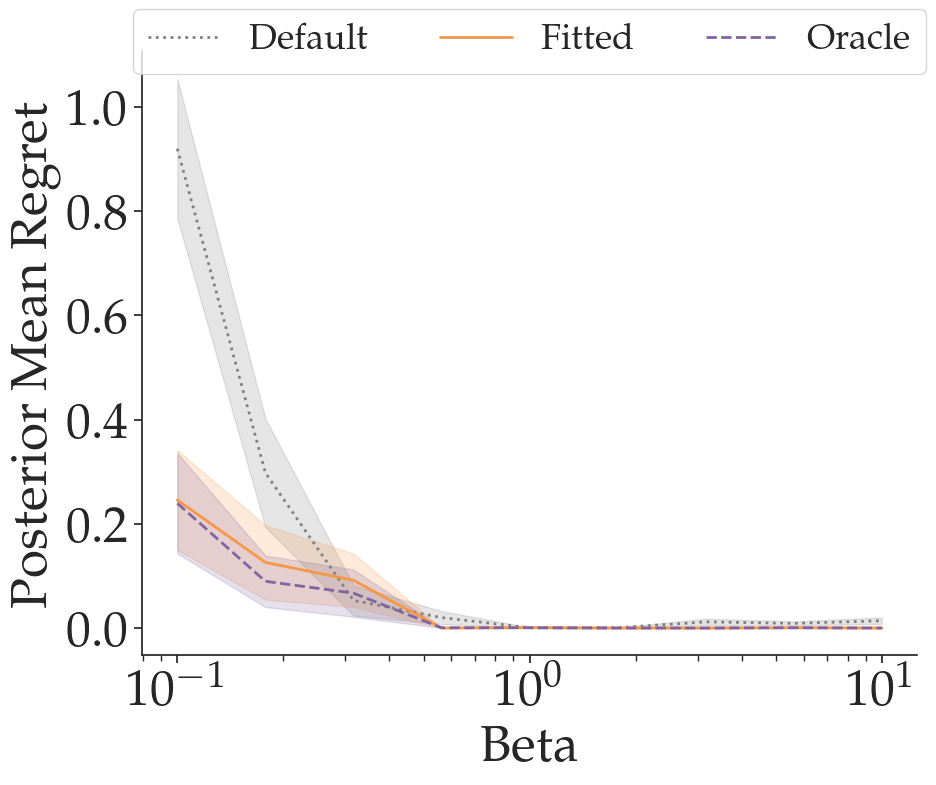}
         \caption{Demonstrations}
         \label{subfig:demo_rational}
     \end{subfigure}
     \begin{subfigure}[b]{0.32\textwidth}
         \centering
         \includegraphics[width=\textwidth]{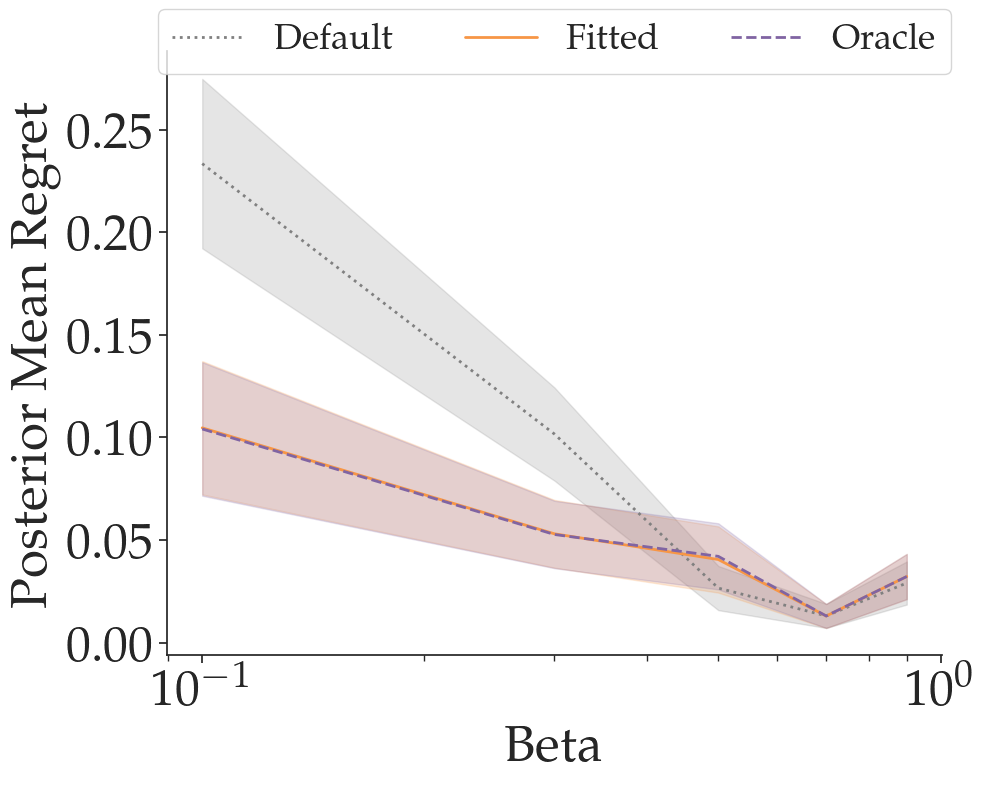}
         \caption{Comparisons}
         \label{subfig:comp_rational}
     \end{subfigure}
     \begin{subfigure}[b]{0.32\textwidth}
         \centering
         \includegraphics[width=\textwidth]{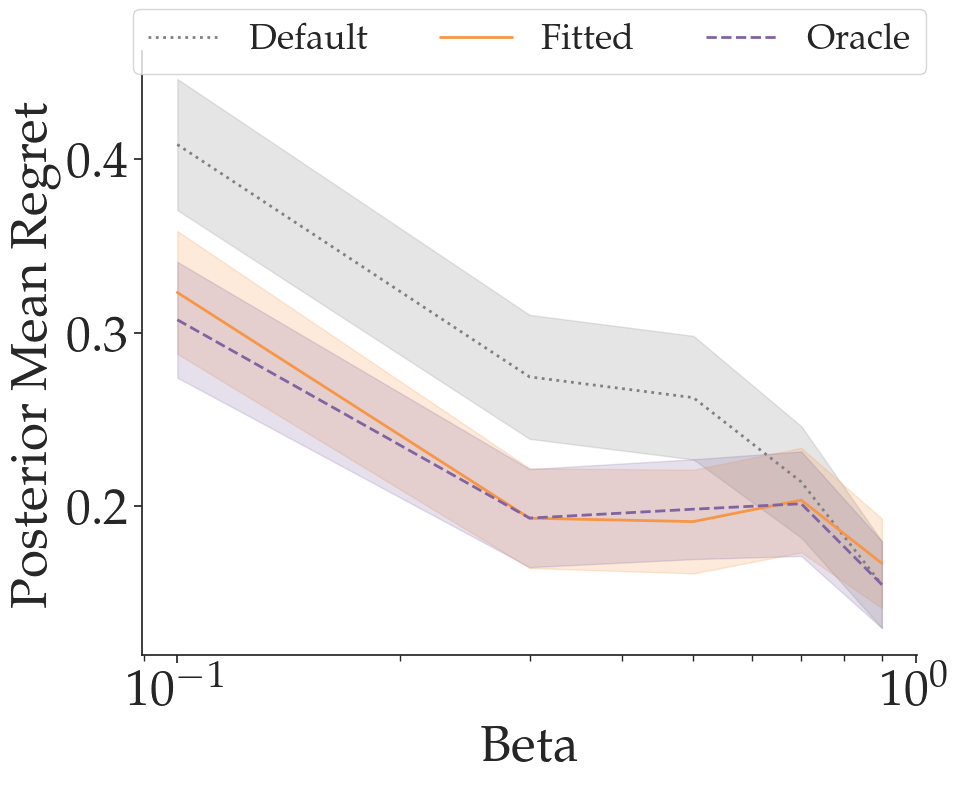}
         \caption{E-Stops}
         \label{subfig:estop_rational}
     \end{subfigure}
         \caption{ \textbf{Reward Inference on Simulated Boltzmann-Rational Feedback.} We compare the rationality-fitting (Fitted) method which learns $\hat{\beta}$ from data to a Default method which assumes $\hat{\beta}=1$ and an Oracle method which sets $\hat{\beta} = \beta^*$. Results show that Fitted (our approach) significantly outperforms Default, achieving similar performance to Oracle. Note that due to the overlap between fitted (orange) and oracle (purple) both curves appear red.
         }

    \label{fig:boltzmannfeedback}
\end{figure*}
\section{Effect of \texorpdfstring{$\hat{\beta}$}{Beta} when Learning from a Single Feedback Type}
In this section, we test the following hypothesis:

\vspace{1mm}
\noindent \textbf{H1:} \textit{Reward inference with a fitted beta will perform better than reward inference with a default beta across different feedback types and different forms of human irrationality.}

We test hypothesis \textbf{H1} over different forms of simulated feedback: first Boltzmann-rational and then systematically biased. In Section~\ref{User}, we test \textbf{H1} on real human data.

\vspace{1mm}
\noindent\textbf{Metrics.}
We measure reward inference performance via the normalized regret from optimizing the posterior mean reward i.e., the difference measured in the ground truth reward between optimizing the ground truth reward vs optimizing the posterior mean inferred reward. In Appendix C, we evaluate other reward inference metrics, which show similar trends.

\vspace{1mm}
\noindent\textbf{Experimental Design.}
We follow the same experimental design when fitting $\hat{\beta}$ and performing reward inference for both simulated and real human feedback. When fitting $\hat{\beta}$, we use 4 randomly chosen calibration reward functions and query the human for feedback 5 times for each calibration reward. When performing reward inference we query the human for feedback 5 times and then perform reward inference using the previously calibrated $\hat{\beta}$. When analyzing learning from individual feedback types, we randomly sample designs. In \secref{sec:active_learning} we examine actively selecting designs.

\vspace{1mm}
\noindent\textbf{Experimental Domains.}
Our simulation experiments take place in a suite of discrete gridworld navigation environments. Reward functions, $r_\theta$, are parameterized by a linear combination of features that indicate the color of each gridcell (see Appendix B for more details). In order to perform exact Bayesian inference, we discretize the reward space by 1000 points. Additionally, in Appendix K, we provide evidence of similar results in a self-driving car domain.

\subsection{Learning from Simulated Boltzmann-Rational Feedback}\label{sec:simexpboltz}
\noindent\textbf{Results.}
We assess the importance of beta fitting for demonstrations, comparisons, and e-stops by running reward inference on simulated Boltzmann-rational feedback generated with different $\beta$ values. We compare reward inference using the fitted $\hat{\beta}$ (\textit{Fitted}) to a default method that sets $\hat{\beta}=1$ (\textit{Default}) and an oracle method that has access to the ground truth $\beta$ value (\textit{Oracle}). The results in \figref{fig:boltzmannfeedback} demonstrate that when feedback is highly sub-optimal, \textit{Fitted} results in significantly better inference than \textit{Default}, and performs comparably to \textit{Oracle}. These observations support \textbf{H1}---using a learned value for $\hat{\beta}$ improves performance, especially in cases where the human acts more noisily. 

\vspace{1mm}
\noindent\textbf{Remark: Under-estimating $\beta$ is better than over-estimating it.}
In \figref{fig:boltzmannfeedback}, we observe an asymmetry between the settings of over- and under-estimating $\beta$. We see that while over-estimation results in poorer performance, under-estimation does not harm reward inference performance as much. In what follows, we present some intuition for this phenomenon. In particular, we show that using a lower $\hat{\beta}$ is risk-averse when the human is suboptimal but is still a good choice even when the person is optimal, whereas a high value of $\hat{\beta}$ leads to poor reward inference when the human is suboptimal. To simplify notation, we define $r(c) \triangleq \mathbb{E}_{\xi \sim \varphi(x, c)}[r(\xi)]$ for $c \in \mathcal{C}$.
\begin{proposition}
If the human is optimal, then $r^*$ is an MLE estimate for any value of $\hat{\beta} \in [0,\infty)$.
\end{proposition}
\begin{proof}
An optimal human (i.e., $\beta = \infty$) never makes mistakes. Thus, $r^*(c_h) \geq r^*(c), \forall c \in \mathcal{C}$ and
\begin{align}
\begin{split}
r^* \in \arg \max_r e^{  r(c_h)} =& \arg \max_r e^{\hat{\beta} \cdot  r(c_h)}\\
=& \arg\max_r P(c_h | r, \hat{\beta}).
\end{split}
\end{align}
Thus, $r^*$ is an MLE estimate given $c_h$, $\forall \hat{\beta} \in [0, \infty)$.
\end{proof}
Even though the MLE reward may not change, the shape of the posterior distribution over $r$ is strongly influenced by the choice of $\hat{\beta}$. When the human is suboptimal, we want to have robots that hedge their bets, rather than becoming overly confident in their estimate of the true reward function. Prior work proposes the Shannon entropy over the robot's belief $P(r \mid c_h)$ as a quantitative measure of the robot's confidence~\cite{jonnavittula2021know}. Using this definition, we present the following result.
\begin{proposition}
The robot becomes more (over-)confident as $\hat{\beta}$ increases.
\end{proposition}
\begin{proof}
If $\hat{\beta} =0$ and we have a uniform prior, then we have a uniform belief distribution over $r$, resulting in maximum entropy and risk-averse behavior. As $\hat{\beta}$ increases, the posterior distribution concentrates on only a small number of reward functions, resulting in lower entropy and less risk-aversion. To see this note that
\begin{align}
    \begin{split}
    P(r \mid c_h, \hat{\beta}) \propto& \frac{\exp(\hat{\beta} \cdot r(c_h))}{\sum_{c \in \mathcal{C}} \exp (\hat{\beta} \cdot r(c))} P(r)\\
    =& \frac{1}{1 + \sum_{c \neq c_h} \exp (\hat{\beta} \cdot \left( r(c) - r(c_h)) \right)} P(r)
    \end{split}
\end{align}
As $\hat{\beta} \rightarrow \infty$, we have $\exp (\hat{\beta} \cdot \left( r(c) - r(c_h)) \right) \rightarrow 0$ if $r(c) < r(c_h)$ and $\exp (\hat{\beta} \cdot \left( r(c) - r(c_h)) \right) \rightarrow \infty$ if $r(c) > r(c_h)$. Thus, $P(r \mid c_h, \hat{\beta}) \rightarrow 0$ as $\hat{\beta} \rightarrow \infty$ if $c_h$ does not maximize $r(c)$ and $P(r \mid c_h, \hat{\beta}) \propto P(r)$ if $c_h$ uniquely maximizes $r(c)$. If $c_h$ is a non-unique maximizer of $r(\phi(x,c))$, then we have $P(r \mid c_h, \hat{\beta}) \propto P(r)/|\{c : r(c) = r(c_h) \}|$.
\end{proof} 
The above result shows that, as $\hat{\beta}$ increases, the Shannon entropy decreases and the robot places high probability on a smaller set of reward functions, thereby behaving very confidently about its estimate of the reward. Finally, we have the following result for when the human makes a sub-optimal feedback choice.
\begin{proposition}
    If the human makes a suboptimal feedback choice, the likelihood of the true reward, $r^{*}$, decreases exponentially as $\hat{\beta}$ increases.
\end{proposition}
\begin{proof}
For a suboptimal choice $c_{h}$, $\exists c^{*}$ such that $r^{*}(c_{h}) < r^{*}(c^{*})$ and 
\begin{align}
    \begin{split}
    P(c_{h}|r^{*}, \hat{\beta}) &= \frac{\exp{(\hat{\beta} r^{*}(c_{h}))}}{\sum_{c} \exp{(\hat{\beta} r^{*}(c))}}\\
    &\leq \frac{\exp{(\hat{\beta} r^{*}(c_{h}))}}{\exp{(\hat{\beta} r^{*}(c^{*}))}}
    = \exp{(\hat{\beta}(r^{*}(c_{h})-r^{*}(c^{*})))}
    \end{split}
\end{align}
By assumption $r^{*}(c_{h})-r^{*}(c^{*}) < 0$. Therefore, the likelihood decreases exponentially as $\hat{\beta}$ increases.
\end{proof}

This analysis supports our empirical findings and shows overestimating $\hat{\beta}$ can have negative consequences since it makes the robot overconfident in the human data, potentially overfitting to mistakes. On the other hand, using a lower $\hat{\beta}$ leads to more risk-averse behavior (beneficial when the human is suboptimal), while still being optimal (under a uniform prior) when learning from perfectly-rational humans.

\subsection{Learning from Simulated Biased Feedback}\label{sec:simexpbiased}
\begin{figure*}[hbt!]
     \centering
     \begin{subfigure}[b]{0.32\textwidth}
         \centering
         \includegraphics[width=\textwidth]{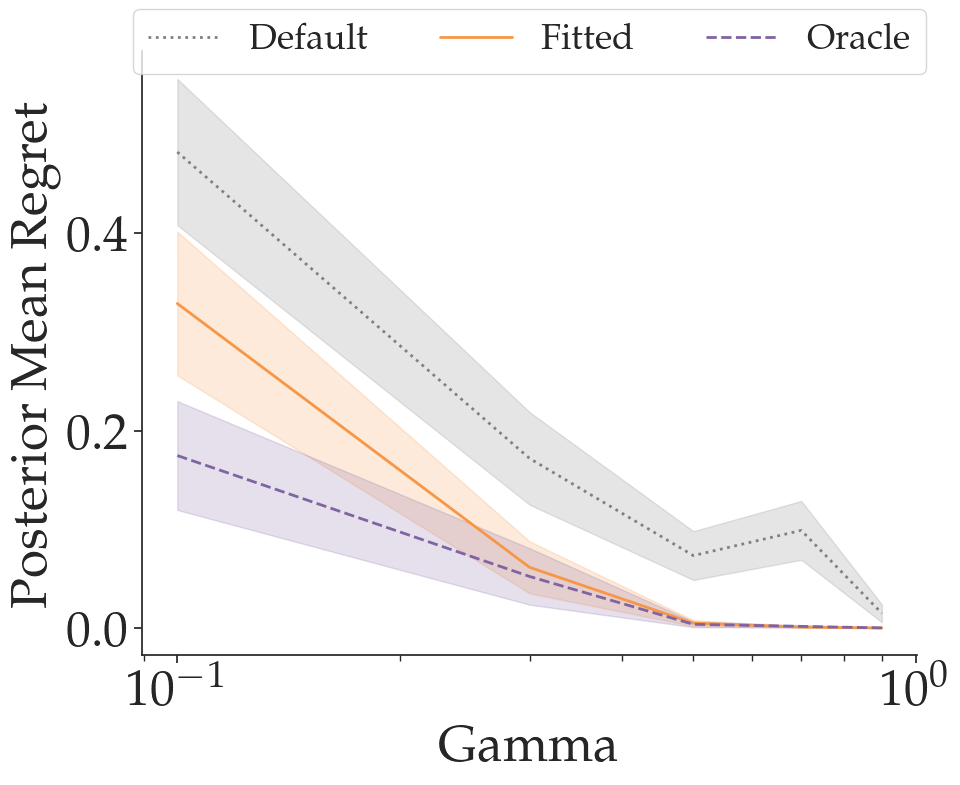}
         \caption{Myopia}
         \label{subfig:myopia}
     \end{subfigure}
     \hfill
     \begin{subfigure}[b]{0.32\textwidth}
         \centering
         \includegraphics[width=\textwidth]{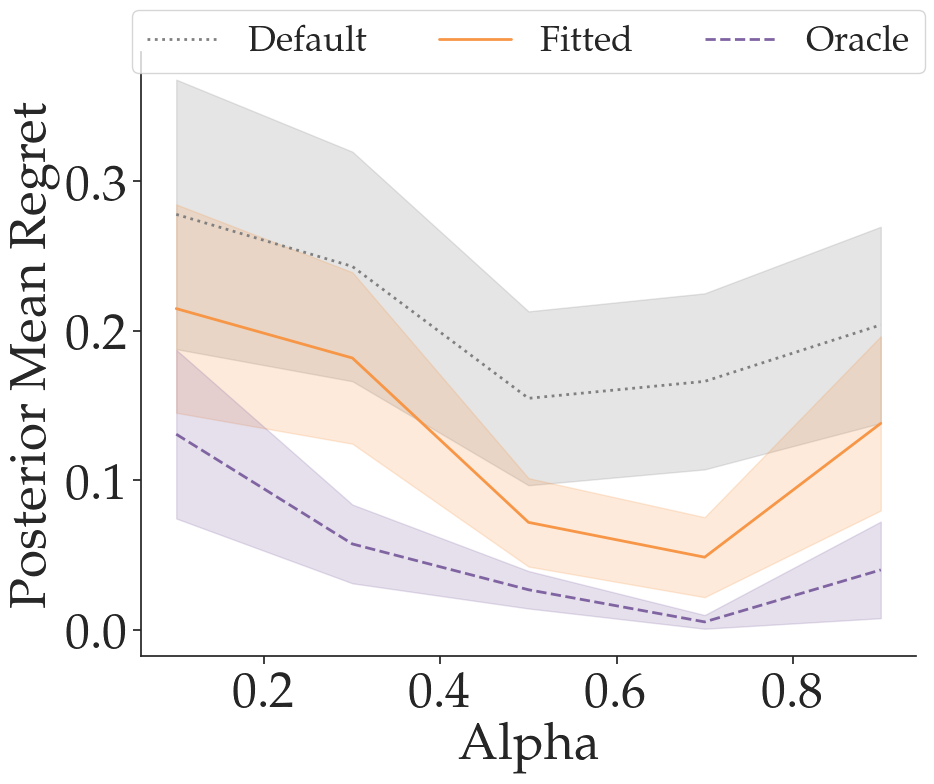}
         \caption{Extremal}
         \label{subfig:extremal}
     \end{subfigure}
     \hfill
     \begin{subfigure}[b]{0.32\textwidth}
         \centering
         \includegraphics[width=\textwidth]{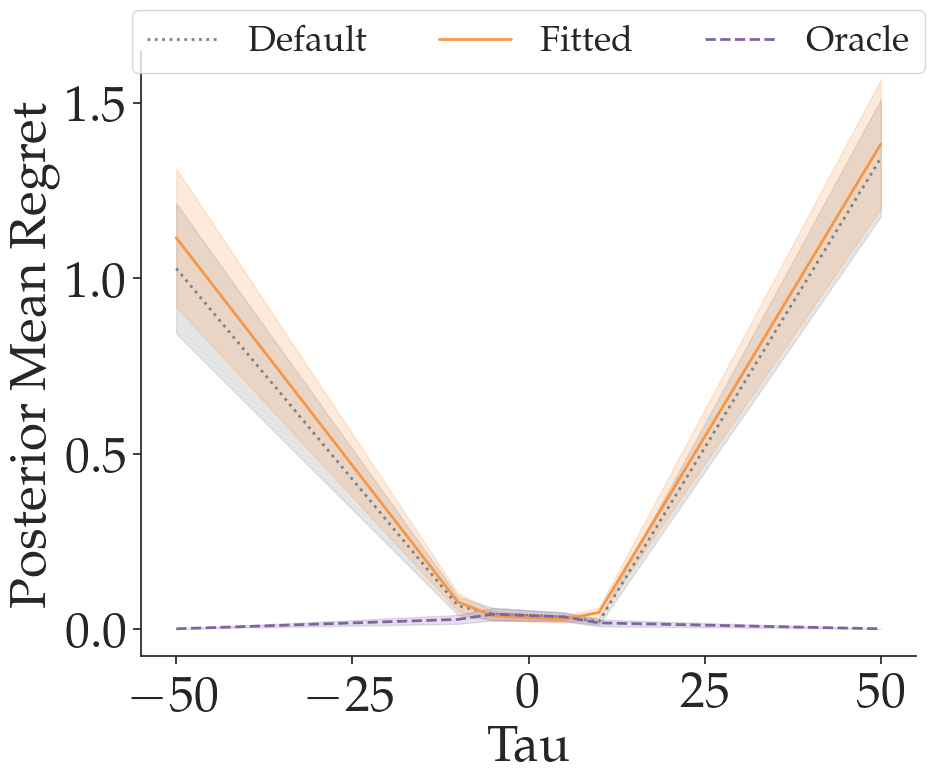}
         \caption{Optimism}
         \label{subfig:optimism}
     \end{subfigure}
        \caption{ \textbf{Reward Inference on Simulated Biased Feedback.} We compare the $\beta$-fitting method (\textit{Fitted}) introduced in this work with a \textit{Default} method which assumes $\beta=1$ and an \textit{Oracle} method that performs reward inference with access to a perfect model of the biased human. Fitted shows improvements over Default for myopia and extremal biases, but shows  no improvement for the optimism bias. 
        }
         \label{fig:biasedfeedback}
\end{figure*}
The results in Section \ref{sec:simexpboltz} demonstrate the importance of modeling a human's rationality level when the human actually is Boltzmann-rational. But of course, human behavior can suffer from systematic biases and irrationalities, not just Boltzmann (ir)rationality~\cite{evans2016learning,alanqary2021modeling}. In this section, we study whether we can use the model of Boltzmann rationality, with a learned rationality level, to improve reward inference in the presence of biases commonly exhibited in human behavior~\cite{guan2015myopic,do2008evaluations,sharot2007neural,thompson1999illusions}. \textbf{H1} hypothesizes that inferring beta can help compensate for these unmodelled aspects of human behavior and therefore lead to better reward inference, in particular when the impact of the bias is not consistent across feedback types. We first evaluate this hypothesis by generating simulated feedback with various biases unknown at reward-inference time. Later we evaluate this hypothesis with a user study. 

\vspace{1mm}
\noindent\textbf{Types of Simulated Bias.}
We study several different models of human biases. Following \cite{chan2021human}, we formalize each bias as a particular modification to the standard Bellman update, resulting in a modified value function which we use to determine the resulting policy and simulate choices from a biased human. We assume that the person is Boltzmann rational under their biased value function. Thus, $\beta$ remains a parameter, and we set it to $1$ in all cases; however, the presence of the bias means that the human is actually \textit{not} Boltzmann $\beta$-rational for any $\beta>0$.

\vspace{1mm}
\noindent
\textbf{\emph{Myopia Bias:}} Humans sometimes demonstrate myopic behavior, concentrating on immediate rewards without evaluating the longer-term impacts of their actions~\cite{guan2015myopic,grune2015models}. We simulate myopic human feedback by changing the discount factor $\gamma \in [0,1]$ and then providing Boltzmann-rational feedback with respect to the value function computed using this discount factor. 

\vspace{1mm}
\noindent
\textbf{\emph{Extremal Bias:}} Humans sometimes pay attention to high-intensity aspects of an experience, at the exclusion of lower-intensity events~\cite{do2008evaluations}. We model this behavior using a modified Bellman update 
\begin{small}
\begin{equation}\label{eq:extremal_bias}
V_{i+1}(s)=\sum\limits_{s' \in S} P(s'|s,a)\max(r(s,a), (1-\alpha) r(s,a) + \alpha V_i(s')),
\end{equation}
\end{small}
where $\alpha \in [0,1]$. As $\alpha \rightarrow 1$ the human seeks to maximize the maximum reward obtained at any point within a trajectory. As $\alpha \rightarrow 0$, the human maximizes immediate reward. 

\vspace{1mm}
\noindent\textbf{\emph{Optimism/Pessimism Bias:}}
Humans can sometimes over- or under-estimate the likelihood of experiencing a good or bad event~\cite{sharot2007neural}. We simulate this bias by changing the transition function that the biased human uses for planning to reflect the fact that the human believes that the likelihood of an outcome depends on its value:
\begin{equation}
    \tilde{P}(s' \mid s,a) \propto P(s' \mid s, a) \cdot \exp\big(\tau \cdot (r(s,a) + \gamma V_i(s'))\big),
\end{equation}
where $\tau \in \mathbb{R}$. As $\tau$ increases (decreases) the human becomes more optimistic (pessimistic).

\vspace{1mm}
\noindent\textbf{Results.}
We study three variants of reward inference: (1) \textit{Default} assumes $\hat{\beta}=1$ and performs inference under the corresponding Boltzmann-rational model of feedback, (2) \textit{Fitted} first fits the rationality parameter $\hat{\beta}$ and then performs reward inference using the $\hat{\beta}$-Boltzmann rational model, and (3) \textit{Oracle} performs reward inference using the true model of the bias.
We present the results for demonstrations in \figref{fig:biasedfeedback} and refer the reader to Appendix D for results on comparisons and e-stops. 

Our findings in \figref{fig:biasedfeedback} provide evidence for \textbf{H1}: using a learned value for $\hat{\beta}$ overall results in lower regret than using a default value. 
However, we find that there is often a large gap between the performance of Fitted and Oracle. This demonstrates that while there is utility in estimating the rationality level of the human, it is not always possible to accurately model systematically biased behavior using a tuned Boltzmann rationality model. In particular, for the Optimism bias, we find that both Fitted and Default perform similarly, and that as the $\tau$ parameter diverges from 0 (diverging from Boltzmann rationality) regret increases.

\vspace{1mm}
\noindent\textbf{Understanding the Success and Failure of Beta Fitting for Biased Human Feedback.}
\begin{figure*}[hbt!]
     \centering
     \begin{subfigure}[b]{0.24\textwidth}
         \centering
         \includegraphics[width=\textwidth]{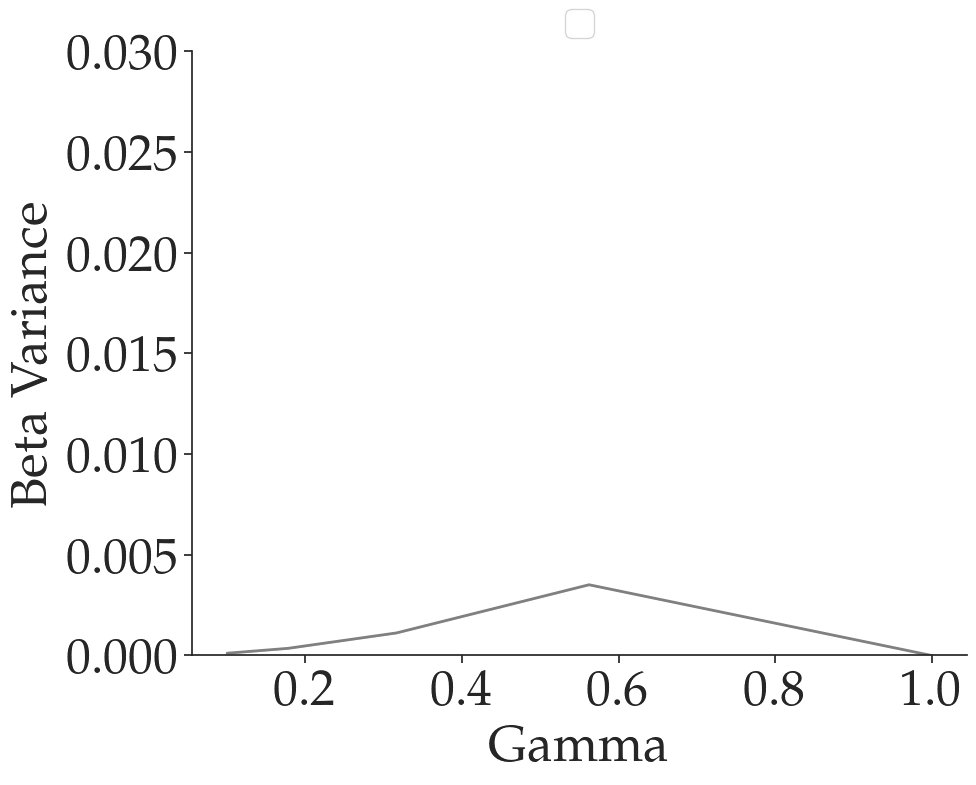}
         \caption{Myopia Beta}
         \label{subfig:myopia_beta_spread}
     \end{subfigure}
     \hfill
     \begin{subfigure}[b]{0.24\textwidth}
         \centering
         \includegraphics[width=\textwidth]{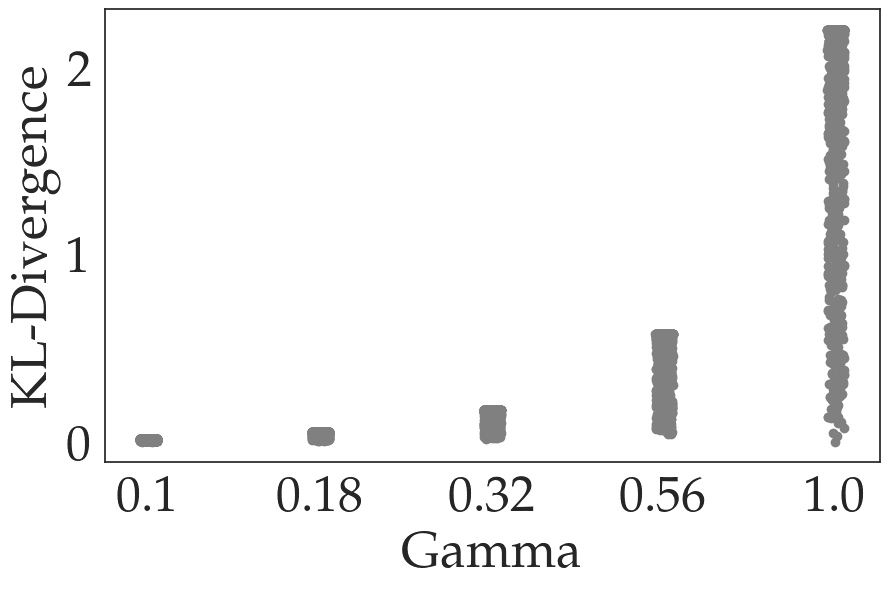}
         \caption{Myopia KL}
         \label{subfig:myopia_kl}
     \end{subfigure}
     \hfill
     \begin{subfigure}[b]{0.24\textwidth}
         \centering
         \includegraphics[width=\textwidth]{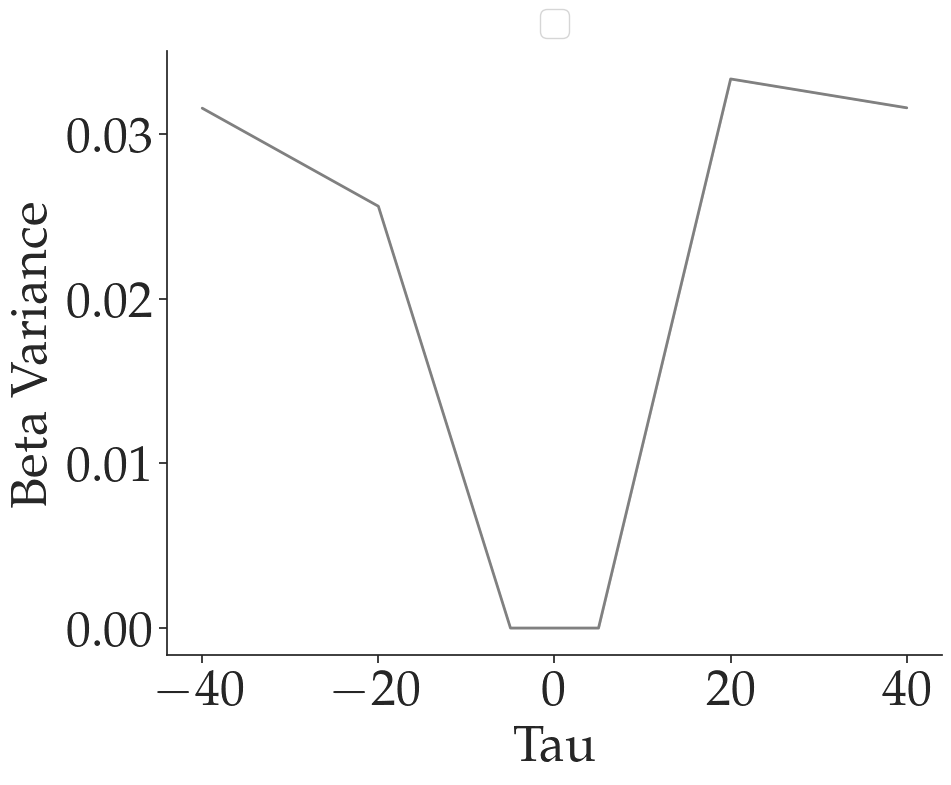}
         \caption{Optimism Beta}
         \label{subfig:optimism_beta_spread}
     \end{subfigure}
     \hfill
     \begin{subfigure}[b]{0.24\textwidth}
         \centering
         \includegraphics[width=\textwidth]{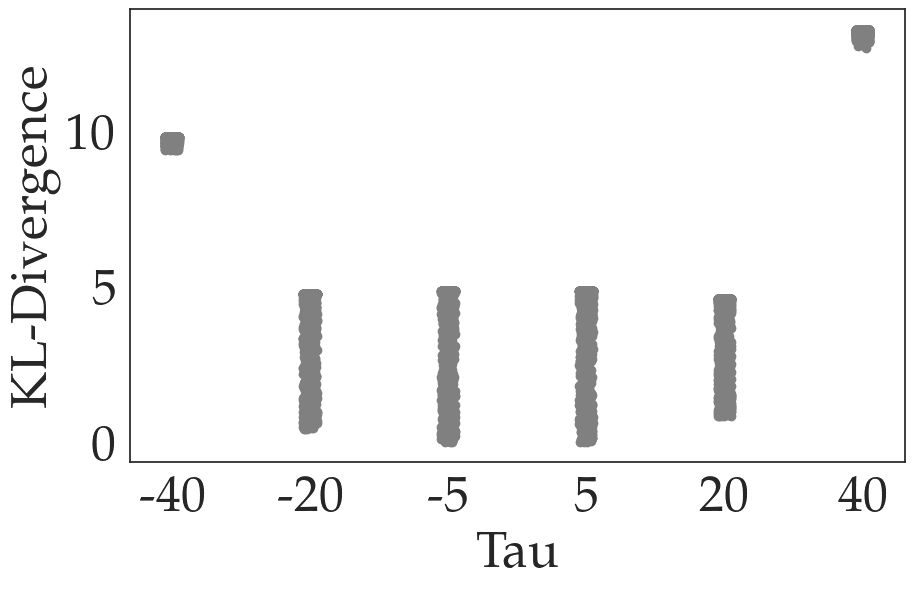}
         \caption{Optimism KL}
         \label{subfig:optimism_kl}
     \end{subfigure}
    \caption{\textbf{Quality of Fit and Generalization of Beta for Biased Demonstrations.} In (a) and (c), we compare the variance of $\hat{\beta}$ over all reward functions in our discretization, when fitting on simulated Myopic and Optimistic demonstration behavior. This indicates how well $\hat{\beta}$ generalizes across reward functions. In (b) and (d), for a fixed reward function $\theta$, we show a scatter plot of the KL-Divergence between the soft-optimal policy for every other reward function $\theta'$ and the Myopic or Optimistic policy on $\theta$. This indicates how well the respective biases are modelled by Boltzmann-rationality and how \textit{identifiable} the ground truth reward $(\theta)$ is.}
    \label{fig:kldivergence}
\end{figure*}
To understand when $\beta$-fitting improves reward inference, we study $\hat{\beta}$-generalization and quality of fit on the Myopia and Optimism biases. We consider the infinite data limit of maximum likelihood estimation, which, as shown in Appendix F,  can be calculated from the biased demonstration policy via an adapted policy evaluation technique.

In Fig.~\ref{fig:kldivergence} (a) and (c), we show the variance of the MLE $\hat{\beta}$ over different reward functions at various levels of Myopia and Optimism bias, respectively. A lower variance in this experiment implies that $\hat{\beta}$ generalizes well across different reward functions for that bias setting. The uniformly low variances for the myopia bias suggest that the fitted $\hat{\beta}$ remains consistent over different reward functions, while the higher variances for the optimism bias suggest that the fitted rationality coefficient, $\hat{\beta}$, is less transferable across different reward functions. 

In Fig.~\ref{fig:kldivergence}
(b) and (d), we fixed the ground truth reward $\theta$ and show scatter plots of the KL-Divergence between the biased policy on $\theta$ and the soft-optimal policies for a large sample of other rewards, $\theta'$. We generate these scatter plots for different bias settings. We observe that the Myopia bias has some rewards with low magnitude KL-Divergences to the biased policy, indicating that the biased policy can be fit well by Boltzmann rationality. On the other hand, for some settings of optimism bias (such as $\tau = \pm 40$), the magnitude of the KL-Divergence is uniformly high, indicating that Boltzmann rationality cannot fit the biased policy well with any reward function. Interestingly, neither a low beta variance nor a low KL-Divergence ensures that $\beta$-fitting can recover the true reward. For example, when $\gamma = 0.1$ both $\beta$ variance and KL-divergence are close to 0, yet fitted performs worse than oracle. 

In Fig.~\ref{fig:kldivergence}, we see that under some bias settings, such as $\tau = \pm 40$ or $\gamma = 0.1$, all the soft-optimal policies for the sampled rewards ($\theta'$) are roughly equidistant from the biased policy (forming a tight cluster in the scatter plot). In these settings, $\beta$-fitting fails since the true reward cannot be uniquely identified---all rewards appear to model the biased behavior equally well. Comparing Fig.~\ref{fig:kldivergence} (b) and (d), we see this situation can arise both when the biased behavior can (in the case of myopia $\gamma = 0.1)$ and cannot (in the case of optimism $\tau = \pm 40$) be modeled well by the Boltzmann distribution, as measured by the scale of the KL-Divergences.
We leave further analysis of which biases preserve reward identifiability to future work.

\section{Effect of \texorpdfstring{$\hat{\beta}$}{Beta} when Actively Learning over Multiple Feedback Types}\label{sec:active_learning}
\begin{figure*}[hbt!]
     \centering
          \begin{subfigure}[b]{.32\textwidth}
         \centering
         \includegraphics[height=.7\textwidth]{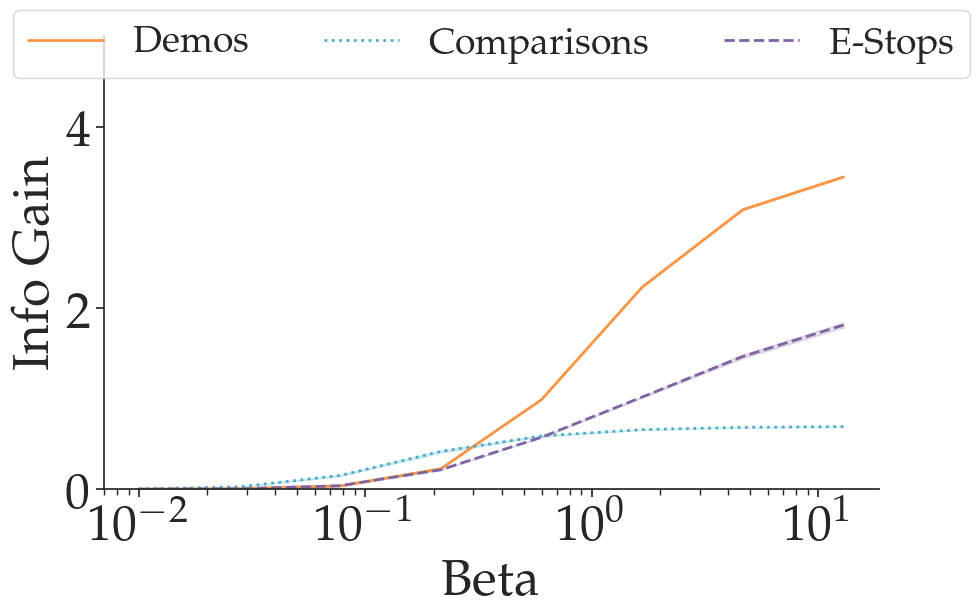}
         \caption{Info Gain vs. Rationality Level}
         \label{subfig:infogaincrossover}
     \end{subfigure}
     \hfill
     \begin{subfigure}[b]{.32\textwidth}
         \centering
         \includegraphics[height=.7\textwidth]{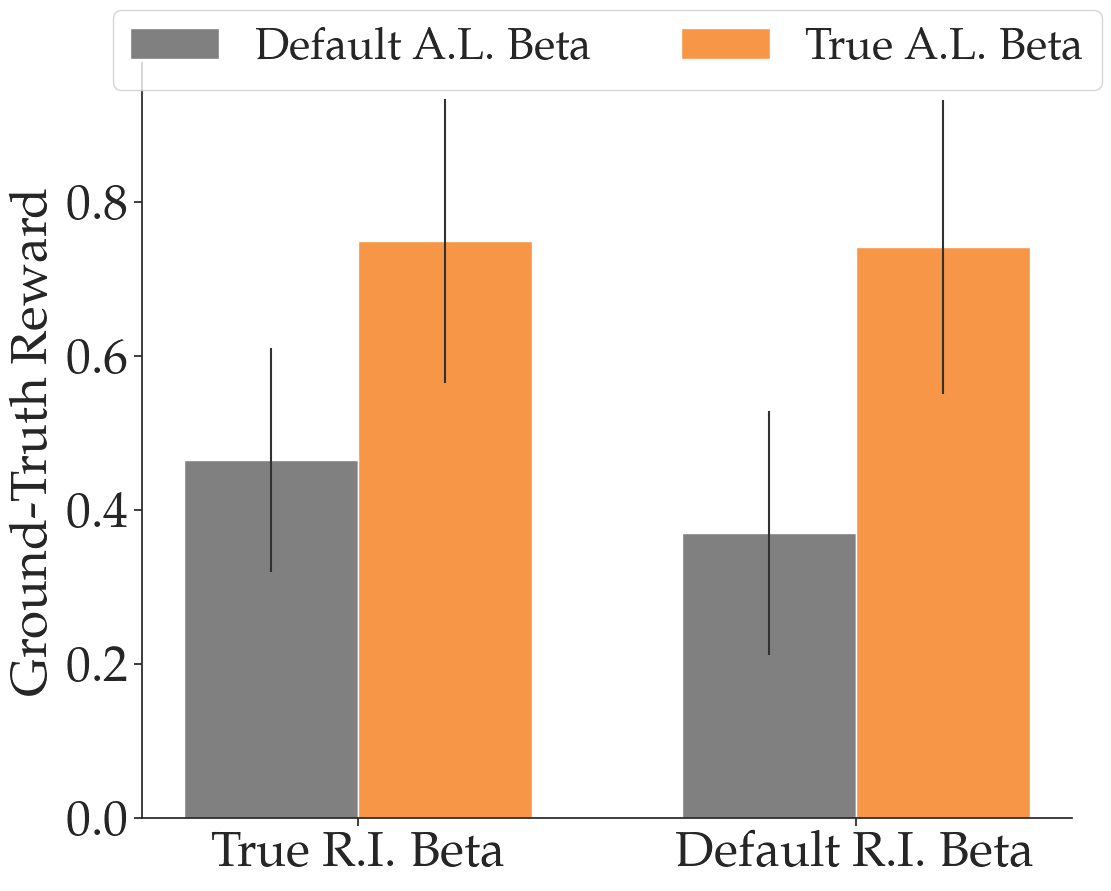}
         \caption{Active Learning Ablation}
         \label{subfig:active_sim}
     \end{subfigure}
     \hfill
     \begin{subfigure}[b]{.32\textwidth}
         \centering
         \includegraphics[height=.7\textwidth]{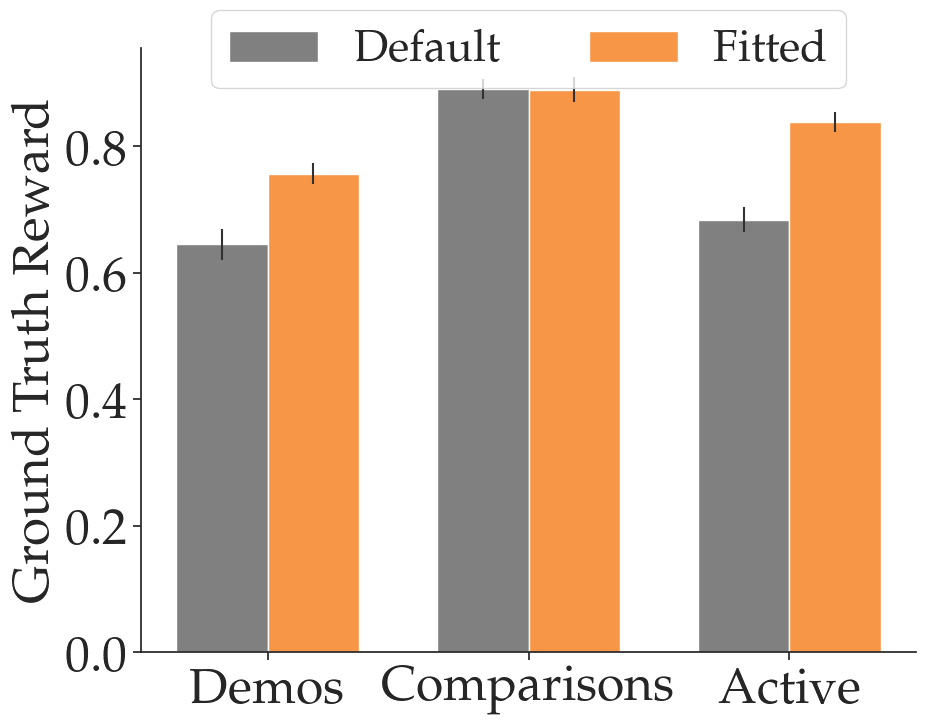}
         \caption{User Study}
         \label{subfig:active_real}
     \end{subfigure}
    \caption{\textbf{Learning Rewards from Multiple Feedback Types.} (a) Given the assumption that the human is $\beta$ rational for all three feedback types, this figure displays the information gain of the most informative design for each feedback type. We observe that one feedback type does not uniformly dominate across all rationality levels, but that the most informative feedback type is a function of human rationality. (b) We study the effect of using a misspecified $\beta$ for active learning (A.L.) and reward inference (R.I.). We find that having access to the correct $\beta$ is particularly important for active learning in this setting. 
    (c) Results from a user study show that $\beta$-fitting helps when learning from demonstrations (designed to be hard) and is no worse for comparisons (designed to be easy) and is especially beneficial when performing active learning over both demonstrations and comparisons.
    }
    \label{fig:ablation}
\end{figure*}

In the previous section, we observed fitting $\beta$ is beneficial when learning passively from a single feedback type. In this section, we consider allowing the robot to actively select the feedback it receives from a set of multiple feedback types, taking into account its current belief over the reward function. When performing this selection, an additional notion of \textit{feedback informativeness} plays a key role. We first investigate the interaction between rationality and feedback informativeness, followed by an analysis of the overall reward inference performance in the active learning setting. Ultimately, we test the following hypothesis: 

\vspace{1mm}
\noindent \textit{\textbf{H2}: Active learning that decides what feedback to ask for will perform better with a fitted beta for each type than with a default beta.}

\vspace{1mm}
\noindent\textbf{Rationality and Feedback Informativeness.}
We first examine the information gain provided by the different feedback types when the human is equally rational across alltypes. In this case, it may appear intuitive that demonstrations would uniformly provide the most information gain because they represent an implicit choice over \emph{all} possible trajectories. However, our results in \figref{subfig:infogaincrossover} reveal that the most valuable human feedback type is a function of the common rationality parameter $\beta$. While demonstrations do provide the most information when the human is highly rational, comparisons gain an advantage when querying a more irrational human. 

In Appendix E we further explore this surprising finding in a toy environment. We construct a toy reward inference environment with a finite set of reward functions and choices. The structure of the reward function and the choice set of this environment means that, given a ground truth reward, only a subset of the choices will be sensitive to correctly identifying this reward. Intuitively, this means that there will be many ``uninformative" choices in the choice set. An uninformative choice will result in poor information gain, as the posterior reward distribution will remain largely unchanged. We model \textit{demonstrations} as a feedback query where the user may choose any choice from the entire choice set. As $\beta \rightarrow 0$, demonstrations become increasingly noisy and the human converges to choosing uniformly from all choices, yielding a higher probability of making an uninformative choice and reducing the expected information gain. On the other hand, \textit{comparisons} restrict the choice set to two elements. Thus, it is possible to construct a comparison query which eliminates uninformative choices and therefore has a higher expected information gain. This analysis confirms the trend we see in \figref{subfig:infogaincrossover}, where the rationality coefficient $\beta$ has a strong influence on the informativeness of different feedback types.




    

\vspace{1mm}
\noindent\textbf{Importance of Beta Fitting for Active Reward Learning.}
In practice, a human is likely to have have varying degrees of rationality across feedback types. Intuitively, $\beta$ plays two roles in this setting: it affects the kinds of queries that are selected as well as the interpretation of the response. In order to gain a more complete understanding of the impact of $\beta$ in this regime, we seek to disentangle the relative importance of these roles. We study the reward inference performance of four variants of active learning, where each of the 2 steps (query selection and reward inference) has either the correct or default $\beta$. In \figref{fig:ablation}, we show these results for one choice of relative rationalities ($\beta_{\text{demo}}=0.1$, $\beta_{\text{comp}}=10, \beta_{\text{estop}}=1$) and in Appendix I, we consider the case where demonstrations are the most rational but all feedback rationalities are overestimated by default. We observe that the relative importance of having the correct $\beta$ for query selection vs. reward inference varies significantly between these cases. When comparisons have a much higher rationality than demonstrations (shown in \figref{fig:ablation}), the value of $\beta$ used for active learning plays a significant role in quality of reward inference, because when comparisons are selected, the default $\beta$ underestimates rationality. On the other hand, Appendix I shows that when the default values of $\beta$ are all overestimates and demonstrations are the most rational, then the $\beta$ used for reward inference has a significant effect on performance. Ultimately, our results reveal an intricate dependency of the optimal active learning strategy on the human rationality and underscores the insufficiency of generic heuristics and the importance of calibrating to individual biases and noise levels.

\section{User Study}
\label{User}
We conducted a small-scale user study of $N=7$ users (aged 21-58, mean = 28) in order to test the effect of conducting $\beta$-fitting on real-world human irrationality. The user study took place in the same grid-world navigation setup as the simulated experiments and each user provided a set of 5 comparisons and 5 demonstrations for each of 5 reward functions. In our setting, the interface for demonstrations was more challenging than that for comparisons, due to the presence of slippery dynamics in the demonstrations control interface (see Appendix J for more details). For each reward function, we tested the reward inference by using the data corresponding to the 4 other reward functions to fit $\hat{\beta}$ and then running reward inference using the individual feedback types, as well as active selection from both. The comparison of performance between using $\beta = 1$ (Default) and $\beta = \hat{\beta}$ (Fitted) are shown in Figure 4(c). We observe that the results validate both \textbf{H1} and \textbf{H2}. $\beta$-fitting on demonstrations results in better performance than using the default $\beta$ and we observe a particularly large benefit from fitting beta in the active learning setting. For comparisons, we observed that the users were able to perform close-to-optimally, which lessened the importance of modeling the rationality level. 

\section{Discussion}\label{sec:discussion}
\noindent \textbf{Summary:} 
In this work, we examine the importance of modeling the level of human rationality when learning from multiple kinds of human feedback. We demonstrate the importance of utilizing the correct rationality coefficient in cases where the human is Boltzmann-rational (with an unknown rationality level), as well in cases where the human is \textit{not} Boltzmann-rational, but is systematically biased. Finally, we demonstrate that $\beta$-fitting is especially important when performing active learning: in a user study we find that active queries based on learned rationality levels significantly outperform an active learning baseline that uses a uniform, default level of rationality across feedback types. 

\vspace{1mm}
\noindent \textbf{Limitations and Future Work:} Our  contribution is studying the importance of having an estimate of $\beta$ (the human's level of rationality), but how exactly to get that remains an open question---our experiments use calibration data, assuming that we can ``incept" a calibration reward function into a human's head and then ask them to provide feedback. We note that this type of calibration approach has been shown to work well in some settings, such as humans interacting with a driving simulator~\cite{schrum2022mind}; however, in other settings, this type of calibration may be difficult, and future work includes studying $\beta$-fitting techniques that do not require providing the human with an explicit calibration reward function. Furthermore, while we study the benefits of $\beta$-fitting on actively learning from multiple feedback types, we have only modeled the rationality level of each feedback type, ignoring the query cost in terms of cognitive burden and time required. Future work includes learning models of the cognitive burden per user and per feedback type, incorporating cognitive and feedback-time costs into active learning, and analyzing $\beta$-fitting in more domains.
\section*{Ethics Statement}
Our work seeks to approximate potentially biased behavior with noisily rational behavior. This could have negative societal impacts if it leads a robot to incorrectly infer human intent, especially in safety critical settings. While our results show that $\beta$-fitting is useful, we caution against simply forcing robots to view all human behavior as $\beta$-rational---using more nuanced and sophisticated models of human bias and irrationality is an important area of future work.
\section*{Acknowledgments}
We thank the members of the InterACT lab for helpful discussion and advice. This work was
supported by the ONR Young Investigator Program (YIP).
\bibliography{references}
\newpage

\appendix

\onecolumn

\section{Expected Information Gain of Design}\label{app:infogain_derivation}
In actively selecting over feedback types, we evaluate feedback designs $x \in \mathcal{X}$ by expected information gain from querying the human for feedback $x$. In Table~1, we provide a reference of the different feedback types, their designs, and choice sets. Concretely active learning over feedback types can be written as:
\begin{equation*}
        \max_{x \in \mathcal{X}} \mathbb{E}_{y \sim P(y|x})[D_{KL}(P(\theta | y, x)||P(\theta)]
\end{equation*}
Using the definition of the KL-Divergence, we have that
\begin{align*}
    \mathbb{E}_{y \sim P(y|x})[D_{KL}(P(\theta | y, x)||P(\theta)]\\ = \mathbb{E}_{y \sim P(y|x)} [\mathbb{E}_{\theta \sim P(\theta|y,x)}[\log P(\theta|y,x) - \log P(\theta)]]\\
    = \mathbb{E}_{y \sim P(y|x)}[-H(P(\theta|y,x))- \mathbb{E}_{\theta \sim P(\theta|y,x)}[\log P(\theta)]]\\
    = \mathbb{E}_{y \sim P(y|x)}[-H(P(\theta|y,x))] - \mathbb{E}_{\theta, y \sim P(\theta, y|x)}[\log P(\theta)]\\
    = \mathbb{E}_{y \sim P(y|x)}[-H(P(\theta|y,x))] + H(P(\theta))\\
    = \mathbb{E}_{y\sim P(y|x)}[H(P(\theta))-H(P(\theta|y,x)]
\end{align*}
Furthermore, we wish to avoid calculating the posterior distribution $P(\theta|y,x)$ since we do not explicitly have access to this distribution. By expanding the entropy term 
\begin{align*}
    \mathbb{E}_{y \sim P(y|x)}[\mathbb{E}_{\theta \sim P(\theta|y,x)}[\log \frac{P(\theta|y,x)}{P(\theta)}]]\\ = \mathbb{E}_{y,\theta \sim P(y,\theta|x)}[\log \frac{P(\theta|y,x)}{P(\theta)}]\\ 
    = \mathbb{E}_{y,\theta \sim P(y,\theta|x)}[\log \frac{P(\theta|y, x)}{P(\theta|x)}]\\
    = \mathbb{E}_{y, \theta \sim P(y,\theta|x)}[\log \frac{P(y|x,\theta)}{P(y|x)}]\\
    = \mathbb{E}_{\theta \sim P(\theta|x)}[\mathbb{E}_{y \sim P(y|x,\theta)}[\log \frac{P(y|x,\theta)}{P(y|x)}]]\\
    = \mathbb{E}_{\theta \sim P(\theta|x)}[\mathbb{E}_{y \sim P(y|x,\theta)}[\log \frac{P(y|x,\theta)}{\sum\limits_{\hat{\theta}}P(y|x,\theta) P(\theta|x)}]]
\end{align*}
Thus, we see that this can be computed given the set of $\theta$ and the observation model for the human feedback $P(y|x,\theta)$ which are both available explicitly.

We observe that this expression is in terms of the observation model for the human feedback $P(c_h|x, \theta)$, along with the prior $P(\theta)$. In the case of comparisons and e-stops, the inner expectation can be computed exactly due to the small space of feedback that may be received from the human. In the case of demonstrations, we approximate the inner expectation by sampling a small number of demonstrations for each possible pair $(x, \theta)$.

\begin{table*}
\label{tab:feedbacktypes}
\footnotesize
\centering
    \caption{Example feedback types and groundings.}\label{tab:feedback_types}
\begin{tabular}{llll}
\toprule 
\bfseries Feedback Type  & \bfseries Possible Designs $\mathcal{X}$   & \bfseries Choice Set $\mathcal{C}(x)$  & \bfseries Grounding $\varphi(x, c)$ \\
 & & for $x \in \mathcal{X}$ &for $x \in \mathcal{X}, c \in \mathcal{C}(x)$   \\
 \midrule 
              
Demonstration & Starting states:               & Available actions:                  & $\varphi(s, a) = ((s,a))$    \\
              & $\mathcal{X} = \mathcal{S}$    & $\mathcal{C}(s) = \mathcal{A}$ &                                                                                         \\
              
\midrule 

Comparison    & \makecell[l]{Pairs of trajectories:\\ $\mathcal{X} =  \Xi^2$}         & \makecell[l]{Which trajectory:\\   $\mathcal{C}((\xi_A, \xi_B)) = \{A, B\}$   }                              & $\varphi((\xi_A, \xi_B), c) = \begin{cases} \xi_A, & c = A \\ \xi_B, & c = B \end{cases}$ \\

\midrule 
E-Stop        &      \makecell[l]{Trajectories: \\ $\mathcal{X} = \Xi$}                          & \makecell[l]{Stopping time: \\ $C(\xi) = \{0, \ldots, T  \}$}                                     &    $\varphi(\xi,t) = \xi_{0:t}$     \\                                                          \bottomrule 
\end{tabular}
\end{table*}
\section{Experimental Setup}\label{app:gridworld}
\subsection{Environment Details}
\begin{figure*}
    \centering
    \includegraphics[width=0.3\textwidth]{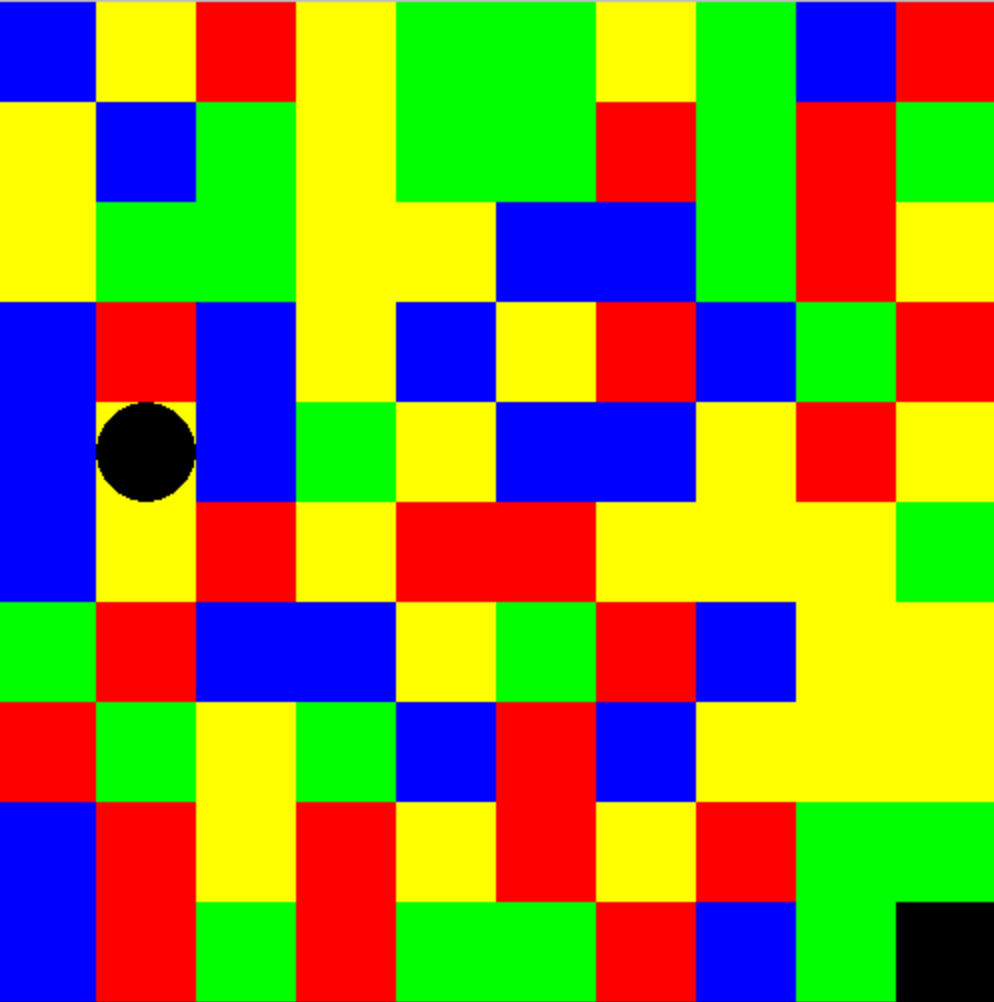}
    \caption{\textbf{Example Colored Grid-world Environment.} Here we show an example environment where our user study and simulation experiments take place. The black circle represents the position of the agent and the black bottom-right corner square shows the goal-state. All other tiles have one of four colors.}
    \label{fig:appendixenvironment}
\end{figure*}
Our user study and simulation experiments take place in a 10 by 10 colored grid-world environment, with each tile having one of four colors and the bottom-right corner being a terminal goal state. An example is shown in Figure \ref{fig:appendixenvironment}. We consider a space of reward functions that assign each tile color to a reward earned for landing on that color. For simplicity, we consider the space of normalized reward functions; the reward functions have unit $\ell_{2}$ norm when considered as four-dimensional vectors $\begin{bmatrix}R_{\textrm{red}} \\ R_{\textrm{blue}}\\R_{\textrm{green}}\\ R_{\textrm{yellow}} \end{bmatrix}$.

\subsection{Reward Inference Performance Metrics}
In this work, we consider two metrics for the success of reward inference metric, as described below: Here $R_{\textrm{inferred}}$ refers to the posterior mean reward and $\theta_{\textrm{inferred}}$ refers to the posterior mean reward vector.\\
\paragraph{Reward Regret} The final goal of reward inference is often to train an agent to maximize the human's reward function. In this setting, the actual values of the reward parameters are irrelevant so long as they induce a similar optimal behavior to the ground-truth reward function. We measure the \textbf{normalized regret}-the amount of reward the agent misses by following the inferred reward- as:
\begin{equation}
\textrm{Regret} = 1-\frac{R_{\textrm{inferrd}}-R_{\textrm{random}}}{R_{\textrm{true}}-R_{\textrm{random}}}
\end{equation}
where $R_{\textrm{inferred}}$ is the reward attained by the optimal policy of the inferred reward under the ground truth reward, $R_{\textrm{random}}$ is the reward attained by a random policy under the ground-truth reward, and $R_{\textrm{True}}$ is the reward attained by the ground-truth reward under the ground-truth reward. 
\paragraph{Reward MSE} Alternatively, we might be interested in how closely the inferred reward parameters match those of the ground-truth. In order to measure this, we computed a reward mean squared error by treating the reward functions as vectors in $\theta \in \mathbb{R}^{4}$, $\theta = \begin{bmatrix} R^{\theta}_{\textrm{red}} \\ R^{\theta}_{\textrm{blue}}\\R^{\theta}_{\textrm{green}}\\R^{\theta}_{\textrm{yellow}} \end{bmatrix}$. Then, we can compute \textbf{reward mean-squared-error} as:
\begin{equation}
    \textrm{MSE} = ||\theta_{\textrm{inferred}} - \theta_{\textrm{ground-truth}}||_{2}^{2}
\end{equation}
In the main text of the paper, we focused on the normalized regret metric, but as we show in the Appendix, our observations correlate well to the reward MSE in general.
\paragraph{Error Bars}
All error bars and ribbons corresponded to the standard error of the mean (SEM). The simulation results have the average and standard error of the mean computed over 10 random reward functions and 2 seeds. The user study results have the average and standard error of the mean computed over the 35 trials performed over 7 subjects.

\subsection{Hyperparameters}
For all simulated experiments, we used the same hyper-parameter settings. All trajectories took place on a 10-by-10 Grid-world with a 4-dimensional reward vector, all lying on the unit sphere. There was a 0.1 transition noise present in the environment. For the purposes of maintaining distributions over reward functions, we discretized this unit sphere by a grid of 1000 points. All experiments were averaged over 10 random seeds.

\section{Reward Inference for Boltzmann Rational Human Feedback}\label{app:other_perf_metrics}
In the main text of the paper, we measured the performance of reward inference using the normalized reward regret. However, another metric we can measure reward inference performance with is the mean-squared error (MSE) when considering the reward functions as four-dimensional vectors. In Figure \ref{fig:app_boltzmannfeedback_mse}, we provide the results of beta-fitting using simulated Boltzmann-rational feedback using this MSE metric. We observe that the mean-squared error metric follows the same general trends as the regret seen in the main text. One important exception, however, is that in the case where we underestimate $\beta$, the MSE results show default having a higher MSE, whereas the regret continues to perform comparably to the fitted and oracle methods. We attribute this to the fact that shifting a reward function by a constant can result in the same eventual behavior, but a reward function whose entries are very different from the ground-truth.

\begin{table*}
\label{tab:feedbacktypeserror}
\footnotesize
\centering
    \caption{Mean Error in Fitting Beta for Different Feedback Types.}\label{tab:menerror}
\begin{tabular}{llll}
 \toprule 
\bfseries Feedback Type  & \bfseries Beta Fit Error   \\
 \midrule 
              
Demonstration & 0.054                                                                       \\
              
\midrule 

Comparison & 0.06 \\

\midrule 
E-Stop  & 0.048                 
\end{tabular}

\end{table*}
We also tested the accuracy of the beta fitting procedure across the different feedback types and show the results in Table 2. These results demonstrate that our beta-fitting procedure is highly accurate, and this accuracy explains the benefits seen by fitting beta in our paper.
\begin{figure*}
     \centering
     \begin{subfigure}[b]{0.32\textwidth}
         \centering
         \includegraphics[width=\textwidth]{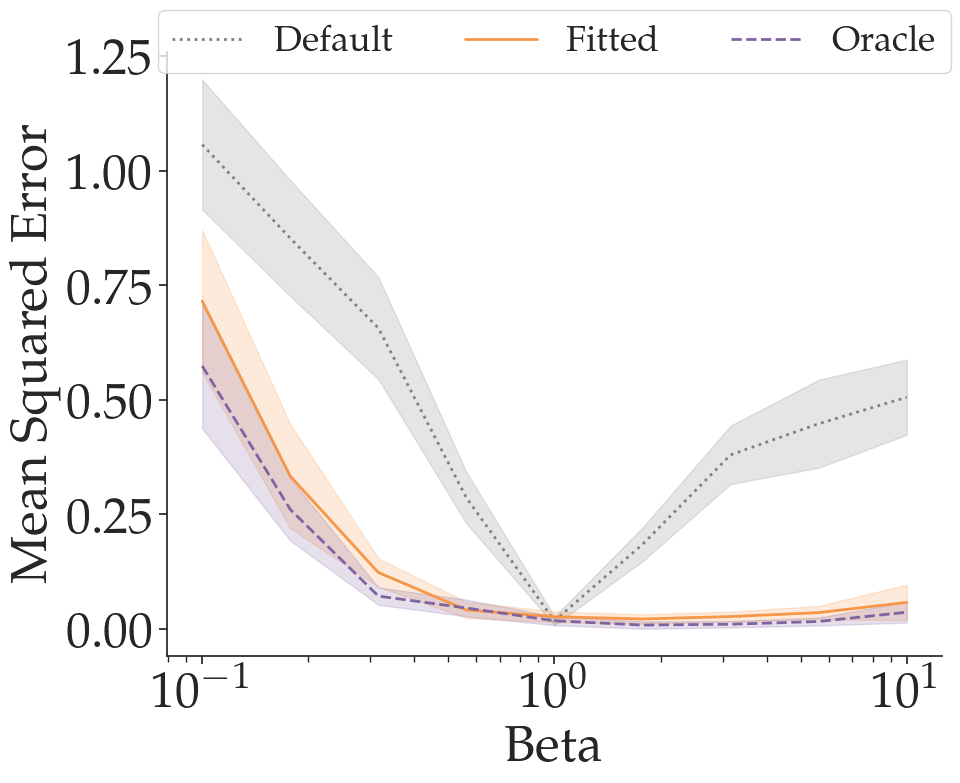}
         \caption{Demonstrations}
         \label{subfig:demorational}
     \end{subfigure}
     \begin{subfigure}[b]{0.32\textwidth}
         \centering
         \includegraphics[width=\textwidth]{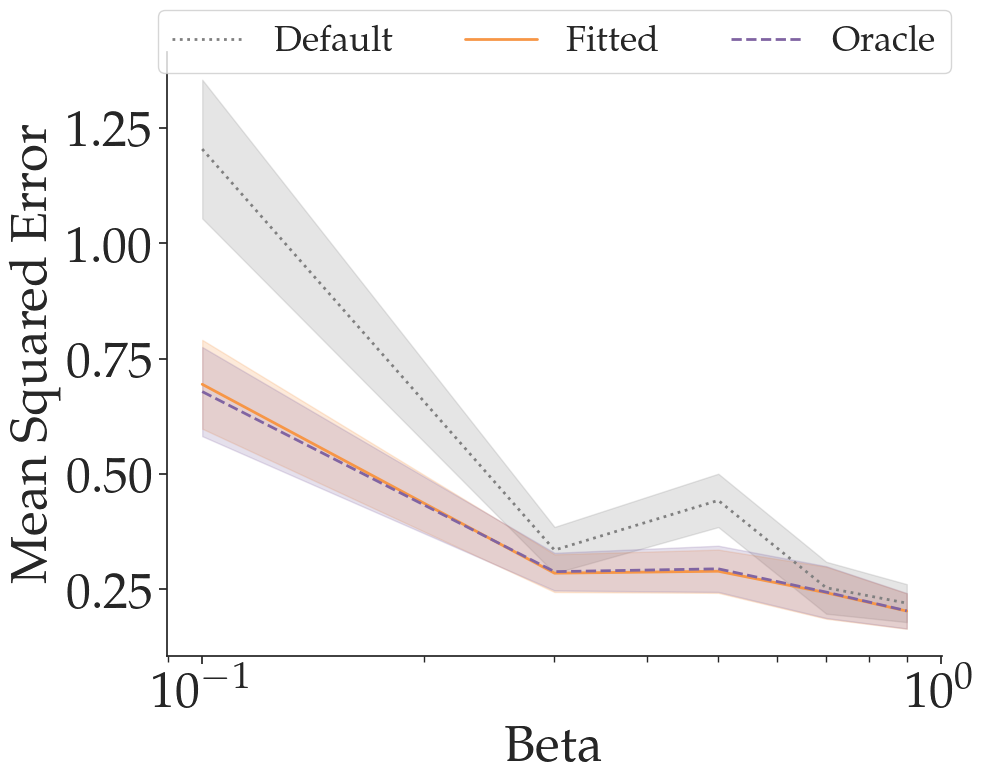}
         \caption{Comparisons}
         \label{subfig:comprational}
     \end{subfigure}
     \begin{subfigure}[b]{0.32\textwidth}
         \centering
         \includegraphics[width=\textwidth]{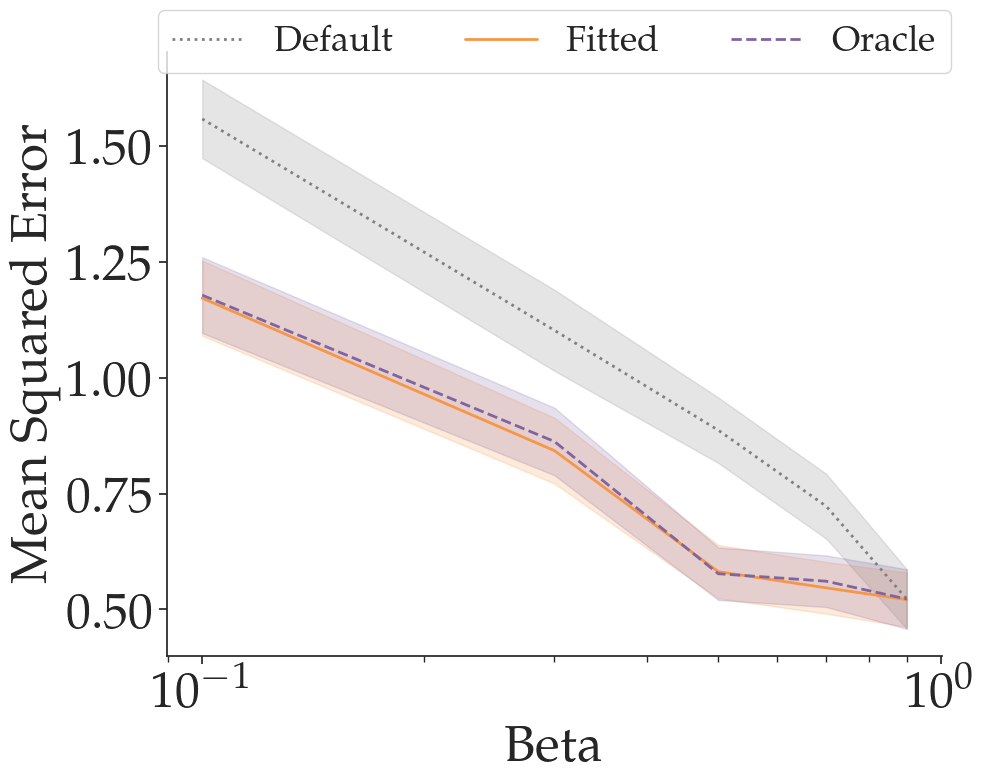}
         \caption{E-Stops}
         \label{subfig:estoprational}
     \end{subfigure}
     \caption{\textbf{Reward MSE on Simulated Boltzmann-Rational Feedback.} We compare the rationality-fitting (Fitted) method which learns $\hat{\beta}$ from data to a Default method which assumes $\hat{\beta}=1$ and an Oracle method which sets $\hat{\beta} = \beta^*$ using the reward mean-squared error metric. Our results generally mirror those using the regret, showing that Fitted (our approach) significantly outperforms Default, achieving similar performance to Oracle. Note that due to the overlap between fitted (orange) and oracle (purple) both curves appear red.}
    \label{fig:app_boltzmannfeedback_mse}
\end{figure*}

\section{Reward Inference for Simulated Biased Human Feedback}\label{app:sim_bias_other_metrics}
In the main text of the paper, we have explored the performance of reward inference on simulated-biased demonstrations. Here, we examine the case of a simulated bias for comparisons and e-stops. In the discussion that follows, it is important to observe that different feedback types are varyingly impacted by different biases. One illustrative case of this is biases which affect the perception of the transition dynamics. Such biases would clearly be expected to impact demonstrations because demonstrations rely upon planning a trajectory through the environment. On the other hand, when confronted with the task of comparing two trajectories having already taken place, transition dynamics play no role in computing the Boltzmann rational choice. Notably, this observation plays an important role in motivating the importance of modelling irrationality. If it is possible for a robot to identify that a human's bias has no impact on a particular feedback type, it has the opportunity to select for that feedback type. In this case, we specifically analyze the case of myopia bias and we leave further analysis of the systematic biases affecting other feedback types to future work.
\paragraph{Multiple Biases} We tested the effect of fitting $\beta$ when there are multiple biases simultaneously for Demonstrations. In particular, we examined the case of a simultaneous myopia and extremal bias with parameters (0.5, 0.5). We found that found that \textbf{Default} achieved regret \textbf{0.37±0.12}, \textbf{Fitted} achieved \textbf{0.11±0.08}, and \textbf{Oracle} achieved \textbf{0.05±0.10}.

\subsection{Comparisons}
\begin{figure*}
     \centering
     \begin{subfigure}[b]{0.4\textwidth}
         \centering
         \includegraphics[width=\textwidth]{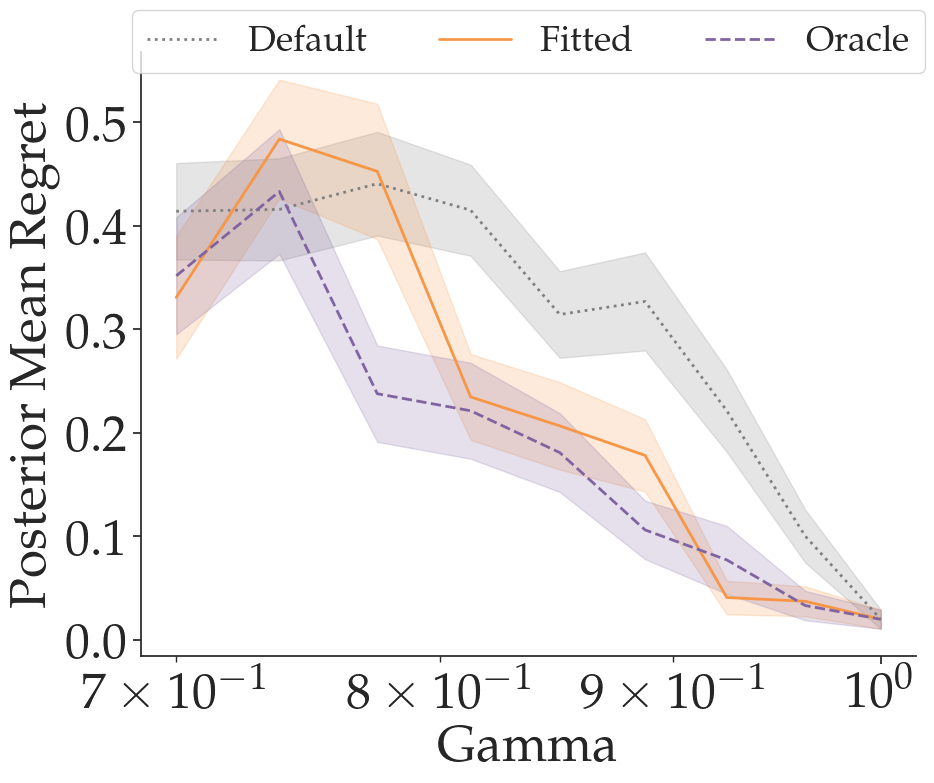}
         \caption{Comparison}
         \label{subfig:app_myopia_comp}
     \end{subfigure}
     \hfill
     \begin{subfigure}[b]{0.4\textwidth}
         \centering
         \includegraphics[width=\textwidth]{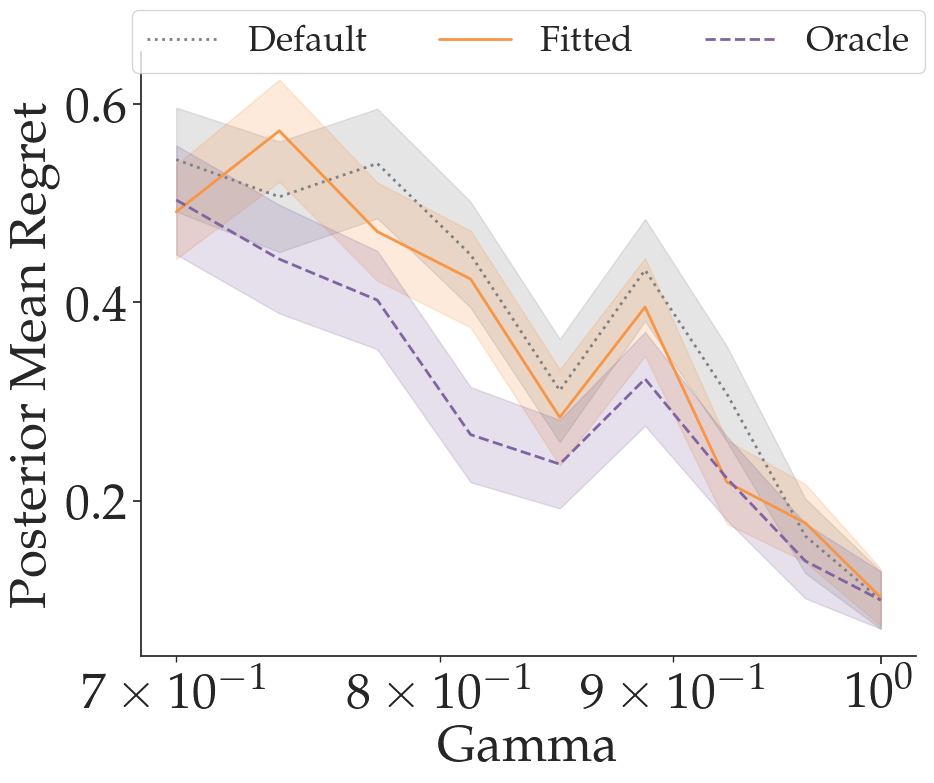}
         \caption{E-Stop}
         \label{subfig:app_extremal_comp}
     \end{subfigure}
        \caption{ \textbf{Reward Inference on Simulated Biased Comparisons and E-Stops.} We compare the $\beta$-fitting method (\textit{Fitted}) introduced in this work with a \textit{Default} method which assumes $\beta=1$ and an \textit{Oracle} method that performs reward inference with access to a perfect model of the biased human. We see in (a) that $\beta$-fitting improves reward inference in the case of comparisons. However, in (b), we find that $\beta$-fitting does not improve performance very much in the case of E-Stops.
        }
         \label{fig:biasedfeedbackcomp}
\end{figure*}
 Recall that the observation model for Boltzmann-rational comparisons as shown in Equation \ref{comp_eq}, involves the quantity $r(\xi) = \sum_{(s,a) \in \xi} r(s,a)$, the cumulative reward achieved during the trajectory. In order to simulate Myopia bias for comparisons, we introduce a discount factor $\gamma$ into this cumulative reward calculation and simulate the comparison choices to be Boltzmann-rational with respect to the discounted cumulative rewards. In Figure \ref{fig:biasedfeedbackcomp} (a), we display the impact of $\beta$-fitting for these biased comparisons. We find that there exists a narrower range of $\gamma$ under which the reward can be recovered well even by the oracle method. Nevertheless, on this range we find the $\beta$-fitting results in substantial improvements to the reward inference.
\subsection{E-Stops}
In order to simulate myopic e-stops, we multiplied a discount factor to the accumulated reward used in the observation model for Boltzmann-rational E-Stops. In Figure \ref{fig:biasedfeedbackcomp} (b), we examine the impact of $\beta$-fitting on myopic e-stops, finding that $\beta$-fitting provides a slight, but not significant improvement over the default method.

\section{Information Gain Crossover Example}
\label{sec:crossover-example}
\begin{figure}
\centering
\includegraphics[width=0.4\textwidth]{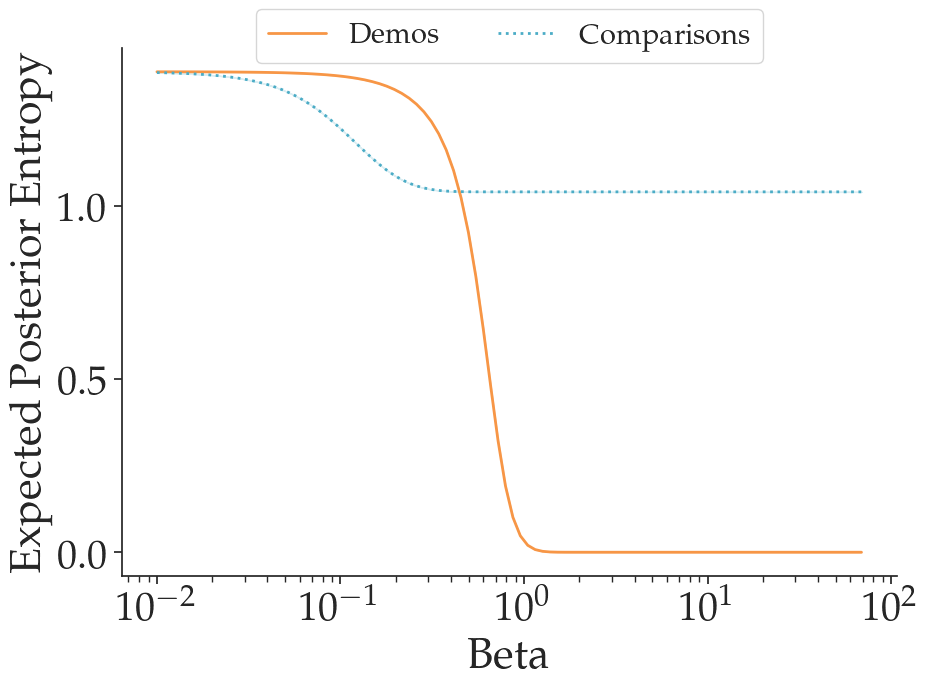}
\caption{\textbf{Example of Posterior Mean Entropy of Comparisons and Demonstrations}. We plot the expected posterior entropy after using comparisons or demonstrations for the example presented in Section E. We observe that there is a sharp transition around $\beta = 1$ from comparisons providing the largest reduction in entropy to demonstrations. We examine the intuition of this phenomena in terms of the particular choice sets involved in Section E.}
\end{figure}
Here is a simple environment determined by a few parameters where different settings of the parameters change whether demonstrations or comparisons are more informative, lending some insight into aspects of the environment which favor different feedback types. The set of possible reward parameters is $\Theta = \{\theta_1^+, \theta_1^-, \dots, \theta_N^+, \theta_N^-\}$, the set of possible choices is $\mathcal{C} = \bigcup_{i=1}^N \{c_i^+, c_i^-, c_i^{p_1}, \dots, c_i^{p_K}, c_i^{n_1}, \dots, c_i^{n_K}\}$, and the rewards of each choice depending on the value of the reward parameter are
\begin{align*}
    r(c_i^s, \theta_j^{s'}) &= \begin{cases}
    R_1 & i=j, s=s' \\
    -R_1 & i=j, (s = +, s' = -)\text{ or }(s = -, s' = +) \\
    R_2 & i=j, (s = p_\ell, s' = +) \text{ or }(s = n_\ell, s' = -)\\\MoveEqLeft\text{ for some $\ell$} \\
    -R_2 & i=j, (s = p_\ell,s' = -) \text{ or }(s = n_\ell, s' = +)\\ \MoveEqLeft\text{ for some $\ell$} \\
    R_3 & i \neq j.
    \end{cases}
\end{align*}
We will set $R_1 > R_2 > R_3 > 0$.
Intuitively, the reward has $N$ different "directions," and along a direction $i$ there are two options $\theta^+_i$ and $\theta^-_i$. Supposing that one of $\theta^+_i$ or $\theta^-_i$ are the true reward parameter, then there are $2(K+1)$ choices which are "sensitive" to whether the reward is $\theta^+_i$ versus $\theta^-_i$, while the remaining $(N-1)2(K+1)$ choices all have equal value either way. Among the $2(K+1)$ choices which are sensitive to $\theta^+_i$ versus $\theta^-_i$, there are two extreme options $c_i^+, c_i^-$ which give large reward $R_1$ if they "match" the reward parameter, and very large cost $-R_1$ if they do not match. There are also a larger number $2K$ of "conservative" options which give smaller reward $R_2$ if they match the reward parameter but also less severe cost $-R_2 $ if they do not match the reward parameter.

Now we analyze the expected information gains from human demonstrations and comparisons in this environment. Demonstrations will allow the human to make any choice, while comparisons will present the human with only two choices. We will only analyze comparisons of the form $c_i^+$ vs $c_i^-$ (for some $i$), but this is hardly a limitation because the resulting information gain provides a lower bound on the information gain of the best comparison. We assume a uniform prior over reward parameters in $\Theta$, and for computational convenience we note that maximizing expected information gain is equivalent to minimizing expected posterior entropy, so we can compute expected posterior entropies for demonstrations and comparisons.

First we compute the expected posterior entropy of a demonstration. After observing a choice $c_i^s$, the posterior mass on $\theta_j^{s'}$ is
\begin{align*}
    P(\theta_j^{s'} \mid c_i^s) & \propto P(c_i^s \mid \theta_j^{s'}) \\
    &\propto \begin{cases}
    \exp \beta R_1 & C_{1} \\
    \exp -\beta R_1 & C_{2} \\
    \exp \beta R_2 & C_{3} \\
    \exp -\beta R_2 & C_{4} \\
    \exp \beta R_3 & i \neq j.
    \end{cases}
\end{align*}
    $C_{1}: i=j, s=s'\\ 
    C_{2}: i=j, (s=+, s'=-)\text{ or }(s = -, s' = +)\\
    C_{3}: i=j, (s=p_\ell, s'=+) \text{ or }(s=n_\ell, s'=-)\text{ for some } \ell\\
    C_{4}:i=j, (s=p_\ell, s'=-) \text{ or }(s=n_\ell, s'=+)\text{ for some } \ell\\
    $
    
For $2(N-1)$ values of $\theta_j^{s'}$, we will be in the bottom case, and there will only be two values of $\theta_j^{s'}$ for which we are in any of the top four cases. If $c_i^s$ was an extreme choice $c_i^+$ or $c_i^- $ then we will reach the top two cases each once, otherwise we will reach cases three and four each once. Therefore if one of the extreme options $c_i^+$ or $c_i^- $ was taken, using $H(p_1, \dots, p_k)$ to denote the entropy of a distribution proportional to $(p_1, \dots, p_k)$, we will obtain posterior entropy $$H\bigl(\exp(\beta R_1), \exp(-\beta R_1), \underset{2(N-1)}{\underbrace{\exp(\beta R_3), \dots, \exp(\beta R_3)}}\bigr).$$
If $c_i^s$ was a conservative choice in $c_i^{p_1}, \dots, c_i^{p_K}, c_i^{n_1}, \dots, c_i^{n_K}$ then the posterior entropy will be
$$H\bigl(\exp(\beta R_2), \exp(-\beta R_2), \underset{2(N-1)}{\underbrace{\exp(\beta R_3), \dots, \exp(\beta R_3)}}\bigr).$$
The probability of the human choosing an extreme choice is proportional to $\exp(\beta R_1) + \exp(-\beta R_1) + 2(N-1)\exp(\beta R_3)$ and the probability of choosing a conservative choice is proportional to $K \exp(\beta R_2) + K  \exp(-\beta R_2) + 2K(N-1) \exp(\beta R_3)$. (Adding these two quantities provides the normalizing constant.) These calculations combine to yield the expected entropy of the human posterior after providing a demonstration.

Now we compute the expected posterior entropy after a comparison. Again, instead of computing the expected posterior entropies of all comparisons and using the best, we only consider comparisons of the form $c_i^+$ vs $c_i^-$, where $i$ is fixed. WLOG we can assume that $c_i^+$ is chosen (since the posterior in the case that $c_i^-$ is chosen is the same up to permutation). The posterior mass on $\theta_j^s$ is
\begin{align*}
    P(\theta_j^s \mid c_i^+) &\propto P(c_i^+ \mid \theta_j^s) \\
    &\propto \begin{cases}
    \exp(\beta R_1) / Z & i=j, s=+  \\
    \exp(-\beta R_1) / Z & i=j, s=- \\
    1/2 & i \neq j
    \end{cases}
\end{align*}
where we define $Z = \exp(\beta R_1) + \exp(- \beta R_1)$. Therefore the posterior entropy is
$$H\bigl(\exp(\beta R_1)/Z, \exp(-\beta R_1)/Z, \underset{2(N-1)}{\underbrace{1/2, \dots, 1/2}}\bigr).$$

We now provide more intuition into these calculations. When an "extreme" choice is taken by the human for a demonstration, we obtain the posterior with the lowest entropy. However, as $K$ increases, it becomes more likely that the human will instead take one of the (less rewarding but more numerous) "conservative" demonstration choices, which result in a worse posterior. This provides an opportunity for comparisons to have a lower expected posterior entropy since the comparison is between some two extreme options $c_i^+$ and $c_i^-$, thus eliminating the conservative and less informative choices. However, for comparisons to be more informative, we need $N$ to be sufficiently low (depending on the values of the other parameters, including $\beta$). This can be verified by examining the entropies calculated above, but an intuitive explanation is that a comparison can only present the two options $c_i^+$ and $c_i^-$  for one particular value of $i$, but the correct value of $i$ (matching the $j$ such that $\theta_j^s$ is the human's reward parameter) to ask for is not known and thus there is only a $1/N$ chance that the robot guesses correctly, and the rest of the time the robot must be presenting the human two options towards which the human is completely indifferent. The robot knows this and thus must temper the formed posterior. When $N=1$ and $K>0$, comparisons are actually strictly superior to demonstrations (for all $\beta$), since the comparison is removing the conservative options and thus forcing the human to make a more informative choice, without being penalized for possibly missing the important direction $i$. When $N>1$, keeping all other parameters fixed, whether demonstrations or comparisons are superior may depend on $\beta$, since it depends on which effect is more harmful: high "dimensionality" causing a large probability of a comparison asking about an irrelevant reward direction, or large $K$ and low $\beta$ causing a large probability of a provided demonstration making a choice which is suboptimal and thus results in a higher-entropy posterior.

\section{Exact Beta Inference}
\label{app:exact_beta_inf}

\subsection{Demonstrations}
When there is access given to the underlying demonstration policy, it is possible to conduct exact inference of the human's rationality.
Because we have access to the entire policy (for a given $\theta$) $\pi$ describing the simulated biased demonstrators, we can circumvent forming $\beta$ estimates by sampling from $\pi$ by instead performing an exact M-projection of $\pi$ onto the family of softoptimal policies $\{\pi_{\beta, \theta} :  \beta \in (0, \infty)\}$ for the particular reward $\theta$:
\begin{align*}
    \hat{\beta} &= \argmin_{\beta \in (0, \infty)} D_{KL}(\pi \mid \mid  \pi_{\beta, \theta}) \\
    &= \argmin_{\beta \in (0, \infty)} \sum_{\xi} \pi(\xi) \log \left( \frac{\pi(\xi)}{\pi_{\beta, \theta}(\xi)}\right) \\
    &= \argmin_{\beta \in (0, \infty)} \sum_{\xi} \pi(\xi) \left( \sum_{t = 0}^{T-1}\log \left( \frac{\pi(a_t \mid s_t)}{\pi_{\beta, \theta}(a_t \mid s_t)}\right) \right) \\
\end{align*}
and the term inside the $\arg \min$ can be evaluated efficiently because it is equivalent to policy evaluation of the policy $\pi$ on the non-stationary reward $r_t(s, a) = \log(\frac{\pi(a_t \mid s_t)}{\pi_{\beta, \theta}(a_t \mid s_t)})$. This is an optimization problem over a single scalar variable so it is easy to solve.

\subsection{Comparisons and E-Stops}
For comparisons and E-stops we take the same approach of performing exact M-projection, but in these cases this is simpler to compute due to the much smaller number of possible choices. We describe this in slightly more generality and then specialize to comparisons and E-stops.

Given $m$ choices $c_1, \dots, c_m$ with rewards $r_{1, \theta}, \dots, r_{m, \theta}$, we can define the $\beta$-rational policy (for a particular $\theta$) $\pi_{\beta, \theta}$ to be the policy which makes choice $c_i$ with probability 
$\exp(\beta r_{i, \theta})/{\sum_{j=1}^m \exp(\beta r_{j, \theta})}$. A general human can be described by the policy $\pi$ which specifies the probabilities of taking each choice. To compute the M-projection we evaluate $
    D_{KL}(\pi \mid \mid  \pi_{\beta, \theta})
$
by simply expanding its definition. This applies to comparisons by taking $m=2$ and to E-stops by having the choices as $c_1, \dots, c_m = \xi_{0:0}, \dots, \xi_{0,T}$ for some trajectory $\xi$. As described this only accommodates a single design (ex. one pair of trajectories to compare/one trajectory to stop along) but the extension to multiple designs is immediate since the choice taken in each design is independent and thus the KL divergence between the true model and a $\beta$-rational model for multiple choices in multiple designs is simply the sum of the individual KL-divergences for each design.

\section{Exact Beta Inference Empirical Results}\label{app:exact_beta_empirical}
\begin{figure*}
     \centering
     \begin{subfigure}[b]{0.4\textwidth}
         \centering
         \includegraphics[width=\textwidth]{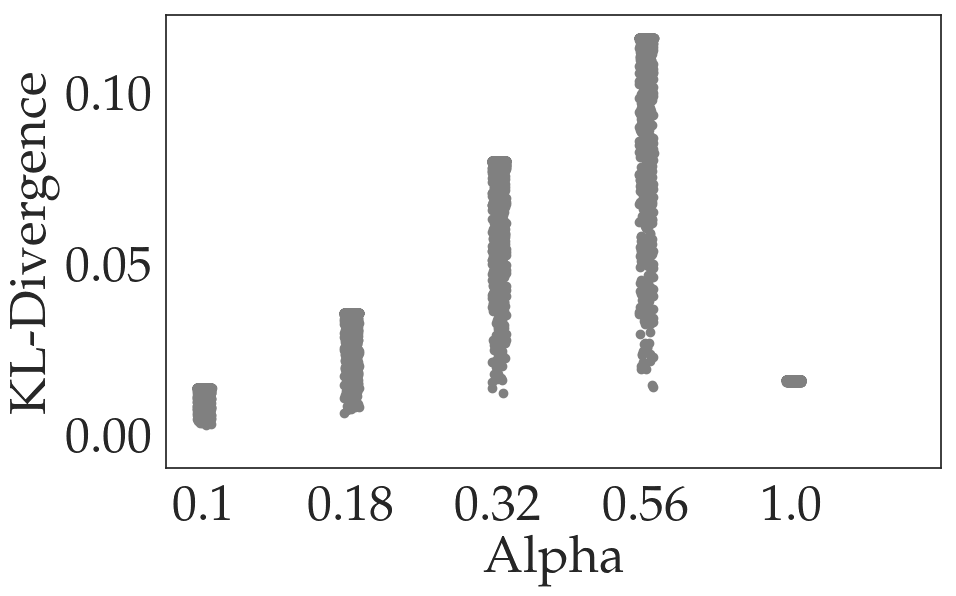}
         \caption{KL Divergence}
         \label{subfig:app_kl_myopia}
     \end{subfigure}
     \hfill
     \begin{subfigure}[b]{0.4\textwidth}
         \centering
         \includegraphics[width=\textwidth]{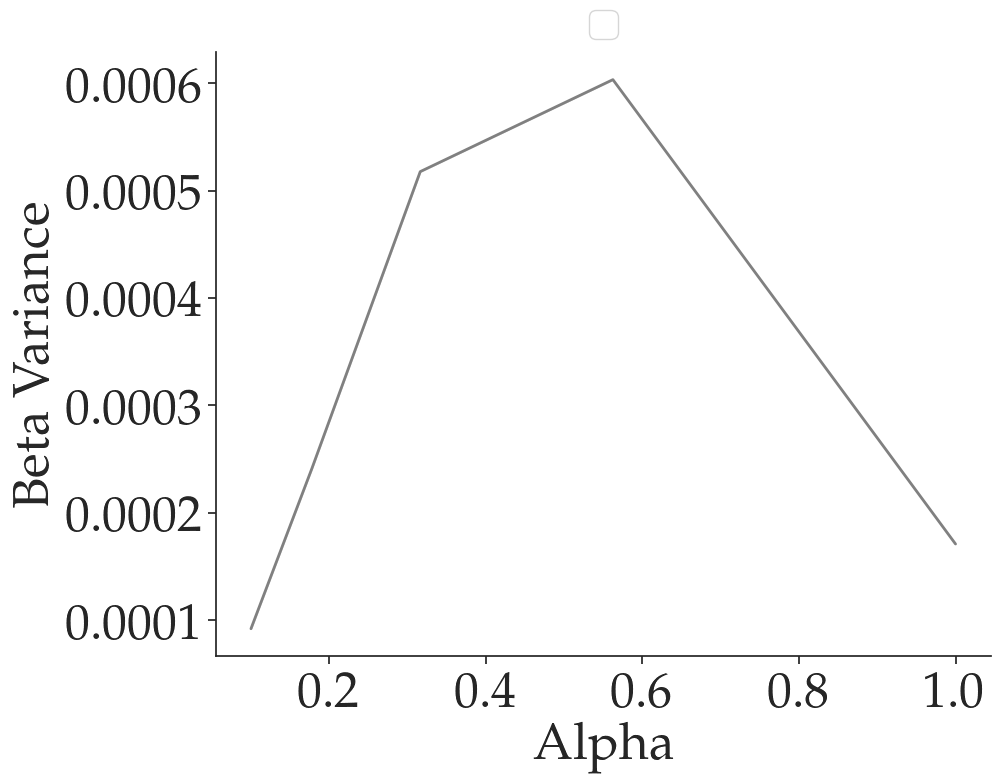}
         \caption{Beta Variance}
         \label{subfig:app_var_extremal}
     \end{subfigure}
        \caption{ \textbf{Exact Inference Analysis for Extremal Bias}. In (a), we show the KL-Divergence of extremal-biased policies against all Boltzmann-rational policies across the candidate reward functions. Our results reinforce the ideas of reward identifiability introduced in the main paper, as the settings where $\beta$ fitting performs comparably to oracle are when the KL-divergences are very spread out. In (b), we show the variance of $\hat{\beta}$ across all reward functions as a measure of $\beta$ generalization. Our variance results do not fully explain the success or failure of $\beta$-fitting because the variance is uniformly low across all bias levels.}
         \label{fig:klextremal}
\end{figure*}
In the main text of the paper, we presented some analysis of the generalization of $\beta$ and the reward identifiability for the myopia and optimism biases. Here we provide a similar analysis for the extremal bias. In Figure \ref{fig:klextremal} (a), we show KL Divergence between the extremal-biased policy and Boltzmann-rational policies for every other candidate reward function. We find that at extremal bias levels where the KL Divergences are spread out, our fitted method performs comparably with the oracle method. On the other hand, at $\alpha$ values where all reward functions have roughly the same KL-Divergence, we find that the fitted method performs worse than oracle. This gives further evidence for the concept of reward identifiability presented in the paper. Additionally, in Figure \ref{fig:klextremal} (b), we presented the variance of $\hat{\beta}$ over all the reward functions at at different levels of extremal bias. We observe that the magnitude of $\beta$ variance is uniformly low. However, we observe the slightly counter-intuitive result that at the parameter values where the $\beta$ variance is lowest, our method does not perform as well as Oracle, showing that the generalization ability of $\beta$ across different reward functions is insufficient to explain the success or failure of $\beta$ fitting. 
\section{Impact of Assumed \texorpdfstring{$\beta$}{Beta} on Learning From a Biased Human}
\textcolor{blue}{\begin{figure*}
     \centering
     \begin{subfigure}[b]{0.4\textwidth}
         \centering
         \includegraphics[width=\textwidth]{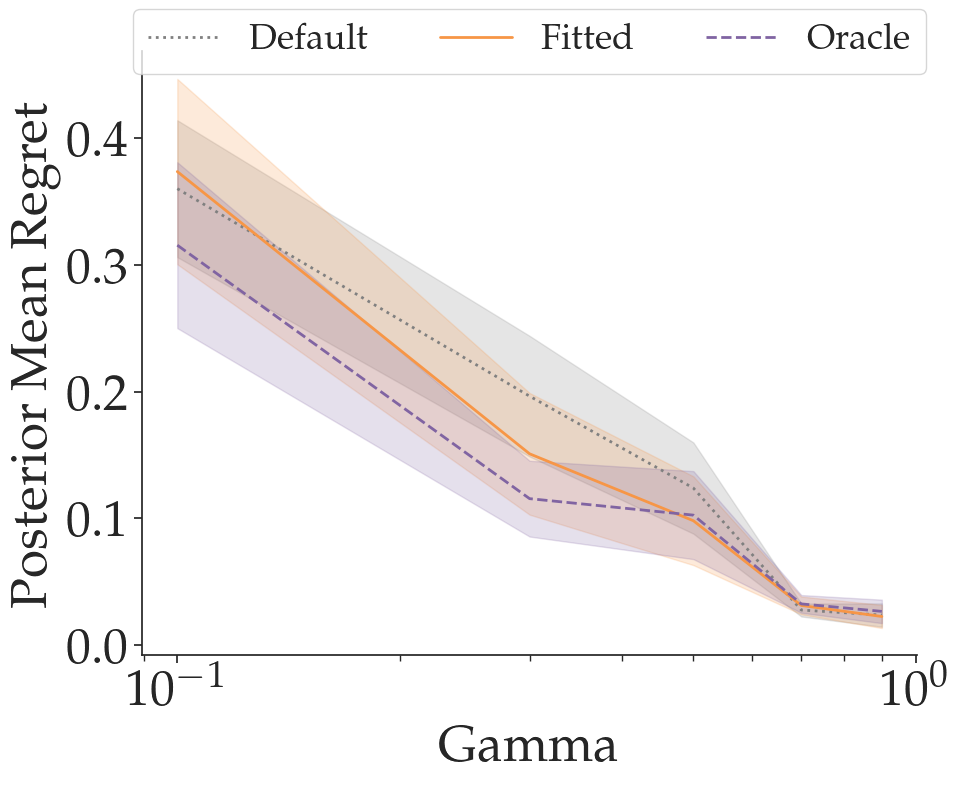}
         \caption{$\beta=0.5$}
         \label{subfig:app_myopia}
     \end{subfigure}
     \hfill
     \begin{subfigure}[b]{0.4\textwidth}
         \centering
         \includegraphics[width=\textwidth]{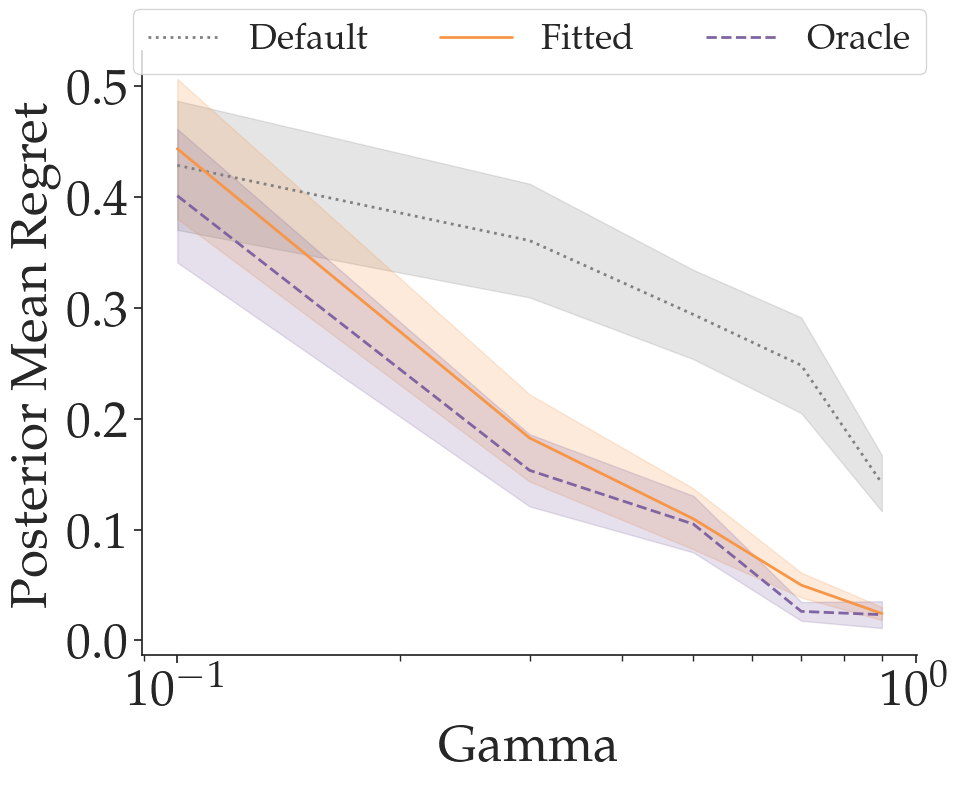}
         \caption{$\beta=5$}
         \label{subfig:app_extremal}
     \end{subfigure}
        \caption{ \textbf{Varying Default Assumed $\beta$}. In this figure, we examine reward inference for myopic demonstrations in the case where the default $\beta$ is considered be (a) 0.5 or (b) 5. We find that increasing the assumed $\beta$ causes the default method to perform worse in comparison to fitted and oracle, while decreasing it narrows the gap between the methods.
        }
         \label{fig:betadefaultvary}
\end{figure*}}

In Figure \ref{fig:betadefaultvary}, we examine the effect of assuming different default $\beta$ values when learning from myopic demonstrations. In one case, we decreased the assumed $\beta$ to 0.5 and found that this caused the default method to perform more comparably to fitted and oracle. In another case, we increased the assumed beta to 5. We found that this caused the default method to perform even more poorly relative to fitted and oracle. These results give further evidence for the theory presented in Section 3: Reducing the assumed $\beta$ causes default to \textit{overestimate} less, leading to improved reward inference. On the other hand, when we increase the assumed $\beta$, default overestimates the true $\beta$, leading to worse reward inference.
\section{Additional Active Learning Ablations}

\begin{figure*}
\centering
\label{fig:secondablation}
     \begin{subfigure}[b]{0.4\textwidth}
         \centering
         \includegraphics[width=\textwidth]{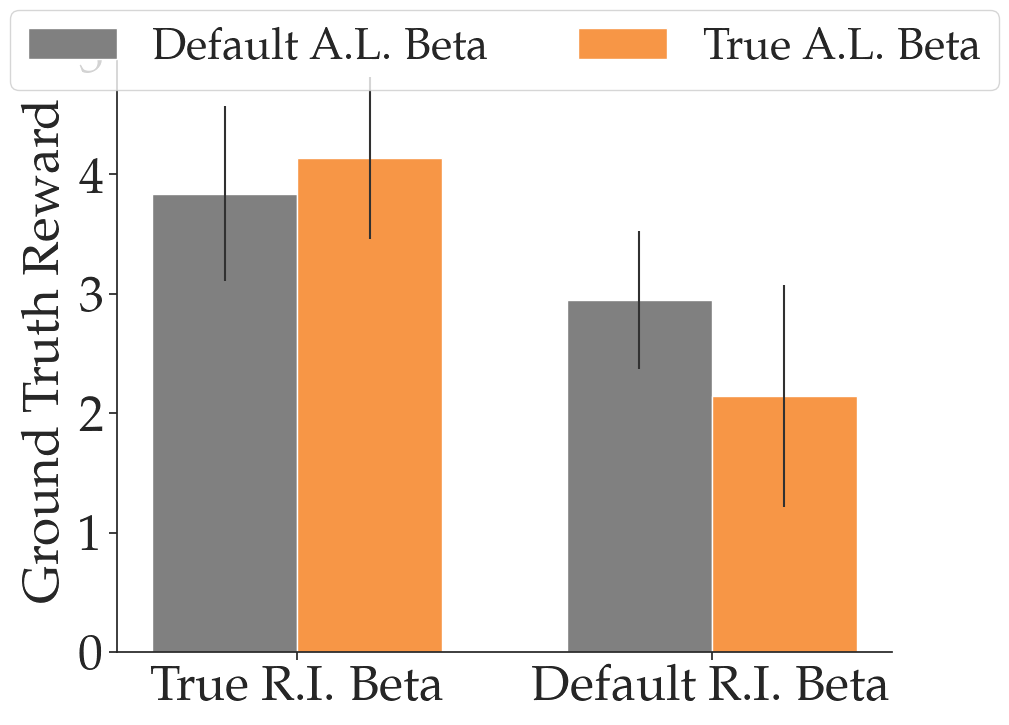}
         \caption{$\beta_{default}=10, \beta_{demo}=1$}
         \label{subfig:app_additional_al_1}
     \end{subfigure}
    \begin{subfigure}[b]{0.4\textwidth}
         \centering
         \includegraphics[width=\textwidth]{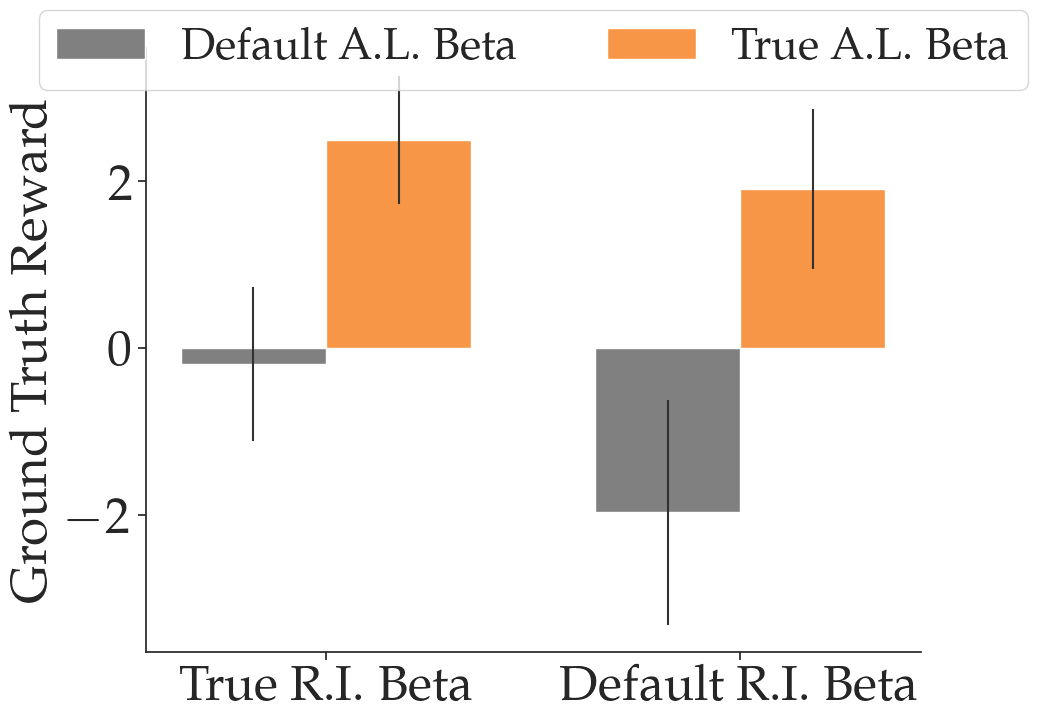}
         \caption{$\beta_{default}=10, \beta_{comp}=1$}
         \label{subfig:app_additional_al_2}
     \end{subfigure}

\caption{\textbf{Additional Active Learning Ablation}. In (a) We additionally consider an active learning ablation in the case that demonstrations are the most rational feedback type and the assumed $\beta$ over-estimates the rationality of every feedback type. In this case, we find that have the correct reward inference beta plays a role of increased importance. In (b), we consider the case where the default $\beta$ overestimates the true rationality and comparisons are the most rational feedback type. Here, we observe that while having the correct active learning $\beta$ is very important (so that you know to choose comparisons), having the correct reward inference beta plays a larger role than the result shown in Figure 4(b), since the default rationality in this case over-estimates the feedback rationality, as opposed to under-estimating it (as is the case in 4(b).}
\end{figure*}
In addition to the active learning setting provided in the main text, we also analyzed the case where (a) the default $\beta$ over-estimates the all true rationalities and (b) the relative rationalities of different feedback types are changed. In order to conduct this study, we set the assumed $\beta$=10 and the true $\beta_{demo}=1, \beta_{comp}=0.5, \beta_{estop}=0.1$. The results are shown in Figure 10. In this case, we find that having the correct $\beta$ for reward inference plays a larger role in determining the performance of the active reward inference. We attribute this to the fact that when the default rationality is used for reward inference, demonstrations are primarily still selected. However, the $\beta=10$ in reward inference severely overestimates the true rationality which causes reward inference to be poor. In Figure 11 (b), we consider the case where the default over-estimates all feedback rationalities and the most informative feedback type is comparisons. In this case, having the correct active learning beta is very important, as this is necessary to correctly choose comparisons queries. However, we see that not having the correct reward inference beta also leads larger to reductions in performance than seen in Figure 4(b). This can be explained by the fact that using the default reward inference rationality in this case causes an overestimation of the feedback rationality, which as shown by our theory can be harmful.

\section{User Study Details}

The user study was performed according to IRB approval and consent was obtained by having the participant review and ask questions about a consent form prior to the study taking place. All data collected from each participant was anonymized. Participants were compensated for their time at a rate of 20 dollars an hour. The user study was designed to take approximately 45 minutes and participants received no less than 15 dollars regardless of the amount of time it took them to complete the study.  In total, \$105 was spent on participant compensation.
\begin{figure*}
     \centering
     \begin{subfigure}[b]{0.3\textwidth}
         \centering
         \includegraphics[width=\textwidth]{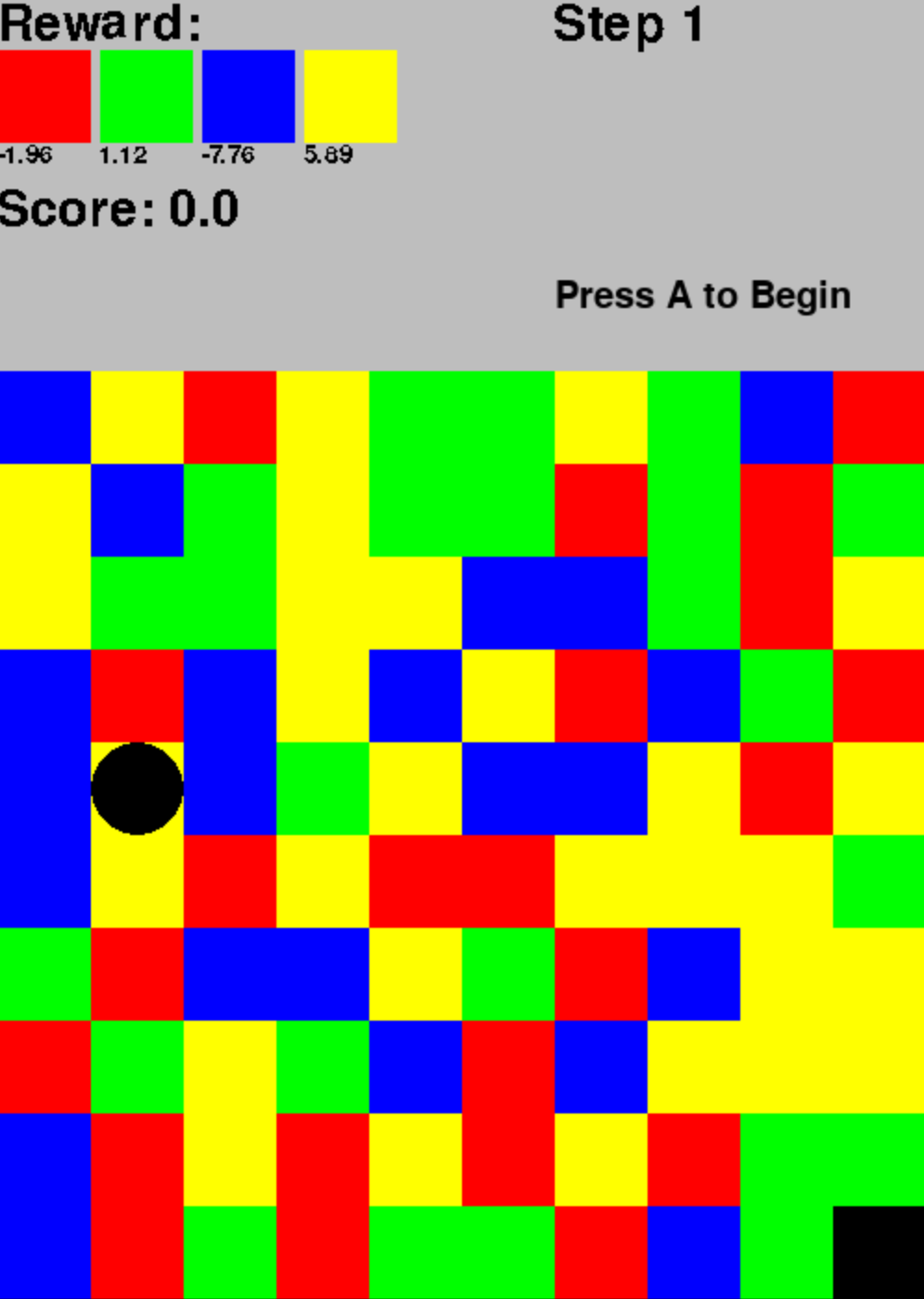}
         \caption{Demonstrations}
         \label{subfig:app_user_demo}
     \end{subfigure}
     \hfill
     \begin{subfigure}[b]{0.6\textwidth}
         \centering
         \includegraphics[width=\textwidth]{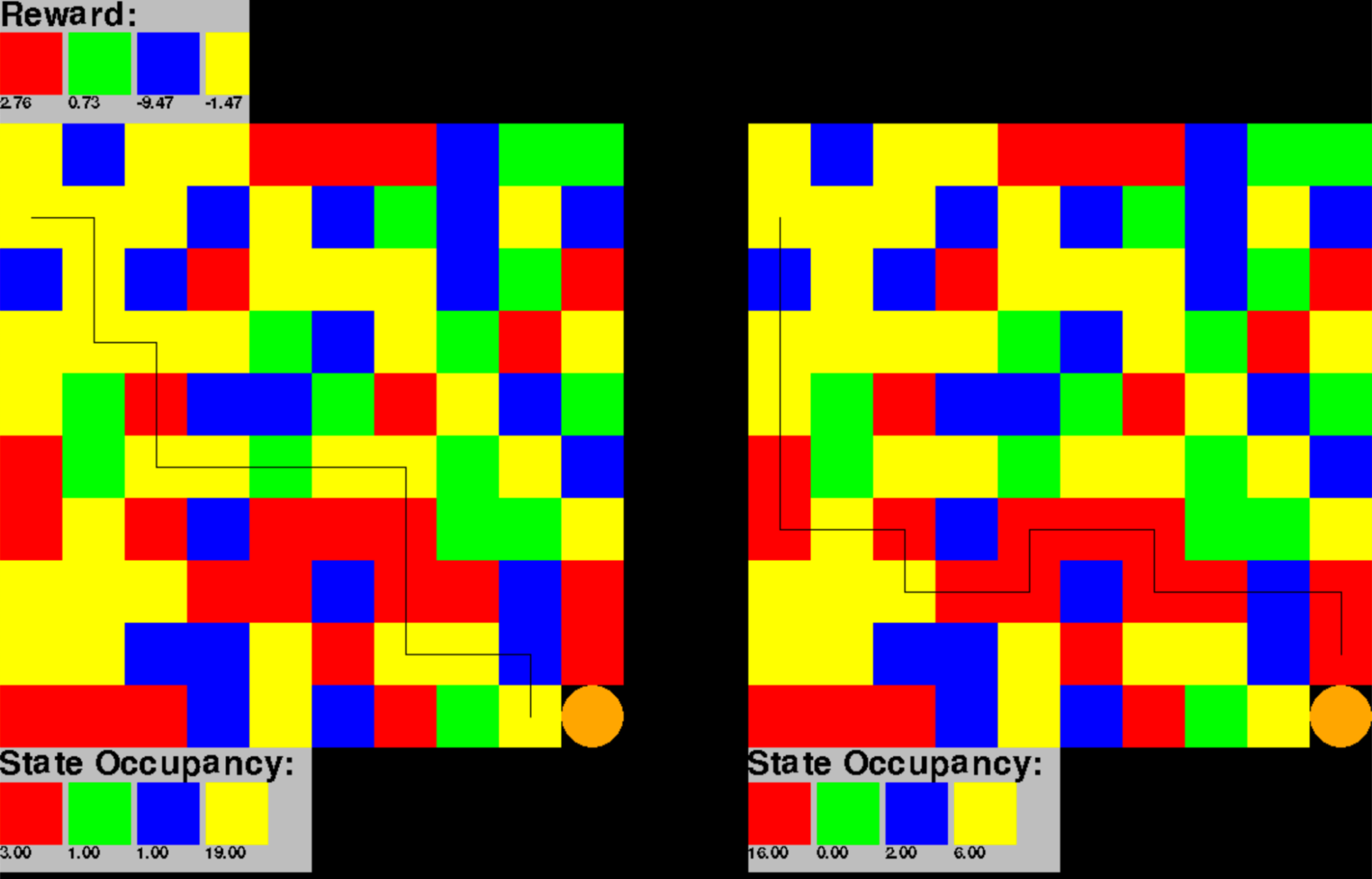}
         \caption{Comparisons}
         \label{subfig:app_user_comp}
     \end{subfigure}
        \caption{ \textbf{User Study Interface}. We show the feedback collection interface for demonstrations (a) and comparisons (b). For demonstrations, the participants could view in real-time the position of the agent, their current score, and the time step. When providing comparisons, participants had access to a trace of the trajectory in the environment, as well as the counts of how many time-steps the agent spent in each colored tile.  
        }
         \label{fig:userstudyinterface}
\end{figure*}

In our user study, we presented 5 calibration rewards to each user (using the interface shown below) and asked them to provide 10 instances of feedback per reward (5 comparisons and 5 demonstrations). In Figure \ref{fig:userstudyinterface}, we provide examples of the feedback collection interfaces for each feedback type. For both feedback types, the reward function was displayed in the top-left hand corner, as a mapping between tile colors and the rewards earned for stepping on the relevant tiles. In addition, there was a completion reward present in the environment of 250 reward and users were advised of this during the instructions phase of the study. The completion reward was chosen to ensure that the optimal path ended in the goal state, while also ensuring that the shortest path from start to goal was not always optimal. 
\paragraph{Demonstration Interface}  Figure \ref{fig:userstudyinterface} (a) shows the demonstration collection interface. In addition to the reward function, participants were provided with real-time access to the current reward score achieved by the demonstration, as well as the current time step. The agent's position was depicted by the black circle and the participant was able to control its movement by using the arrow keys. Participants experienced time pressure when giving the demonstration, as the interface ran at 6 time-steps per second and failure to provide control input at a given time-step would result in the agent continuing to slide in its last direction of movement. The participant was given unlimited time to review the reward function and starting state prior to beginning the demonstration and would press the A key to start the demonstration. 
\paragraph{Comparisons Interface} The participant was presented with two trajectories and asked to compare which was superior under the given reward function. To assist them in making their decision, they were provided with environment maps depicting a trace of the agent's movement. Moreover, participants were provided with time-step counts of how long the agent spent in each color tile. Participants were allowed unlimited time and could indicate their selection by pressing 1 to select the left trajectory and 2 for the right. 

Prior to beginning, participants were provided with the following instructions:

\textit{In this user study, you will be providing two different feedback types to help an artificial intelligence agent infer a reward function. You will be providing demonstrations and comparisons.}

\textit{You will be providing feedback in a 10x10 gridworld environment, where the tiles have four possible colors. You will be controlling the orange circle when you give your demonstrations. 
Each reward function will be a mapping between tile colors and a reward you get for every time step you spend on the tile color. 
The environment will have a horizon length of 25. This means that the maximum number of steps in a trajectory is 25. }

\textit{The black square in the bottom right corner of the environment is a goal state. Provided you reach the goal state by the end of the horizon, you will receive +250 reward. If you arrive at the goal before the horizon ends, you will receive +250 reward at the point you reach the goal and +0 reward for all future time steps.
You will provide demonstrations using an interface that looks like the above. You will control the movement of the agent using the arrow keys (up key moves up, down key moves down, left key moves left, right key moves right). You will experience time pressure when giving demonstrations. The interface will run at 6 time-steps per second and if you don’t provide input at a given step, the agent will slide in the direction it was last moving. 
Score: The score field on the left-hand side will tell you the current score you have accumulated 
Step: The step field on the right-hand side will tell you how many time-steps have elapsed since the beginning of the episode.}

\textit{You are encouraged to review the reward function carefully before starting the demonstration. When you are ready to begin, you should press the “a” key and then begin controlling the robot.
You will also be providing preference comparisons between two trajectories that the robot itself has generated.A trace of the trajectory is shown on the screen. We have also provided information about how long the robot spends on each color to help you decide which trajectory better represents the reward.
}

Following the presentation of the instructions, the participant was given unlimited practice attempts on each feedback type. They were also given a chance to ask any questions they ha about the interface or the strategy for giving the feedback. After this, the participants began the formal study. 

Participants first provided 25 demonstrations, followed by 25 comparisons. The feedback of each type was provided in 5 blocks of 5, where each block consisted consisted of the same reward function. In between blocks, the interface alerted the participants to the fact that the reward function had changed. Since participants had unlimited time to review the reward function before providing a demonstration or indicating a comparison, they were able to pace themselves and take breaks as necessary.

We used a hold-one-out method for analyzing the user study data. In this setting, we would test reward inference on each reward function by fitting beta on the data provided for the other 4 reward functions and then testing the reward inference on the held-out reward function. We completed this for each reward function, resulting in 5 trials per user.
\paragraph{Statistical Tests} We also conducted a two-tailed paired T-Test between the fitted and default reward inference results computed in the User Study. The results are shown in Table 3 and demonstrate that, while there was a statistically significant difference for both active learning and demonstrations, we did not observe a statistically significant difference between the methods for comparisons.
\begin{table*}
\label{tab:app_t_test}
\footnotesize
\centering
    \caption{User Study t-test Results.}\label{tab:userstudyttest}
\begin{tabular}{llll}
 \toprule 
\bfseries Feedback Type  & \bfseries p-Value   \\
 \midrule 
              
Demonstration & 0.02                                                               \\
              
\midrule 

Comparison & 0.34 \\

\midrule 
Active  & 0.0001                
\end{tabular}

\end{table*}
\begin{figure*}
\centering
      \begin{subfigure}[b]{0.28\textwidth}
         \centering
         \includegraphics[width=\textwidth]{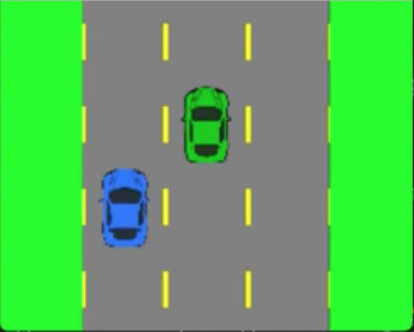}
         \caption{Environment Visualization}
         \label{subfig:app_car_inter}
     \end{subfigure}
     \hfill
     \begin{subfigure}[b]{0.28\textwidth}
         \centering
         \includegraphics[width=\textwidth]{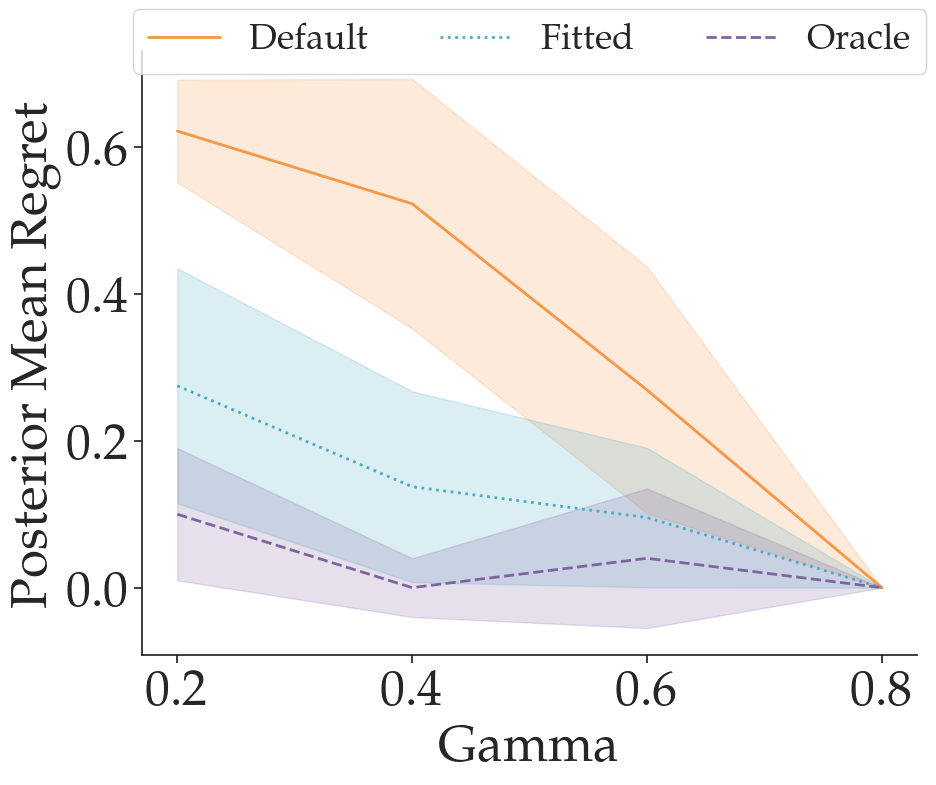}
         \caption{Myopia}
         \label{subfig:app_car_myopia}
     \end{subfigure}
     \hfill
    \begin{subfigure}[b]{0.28\textwidth}
         \centering
         \includegraphics[width=\textwidth]{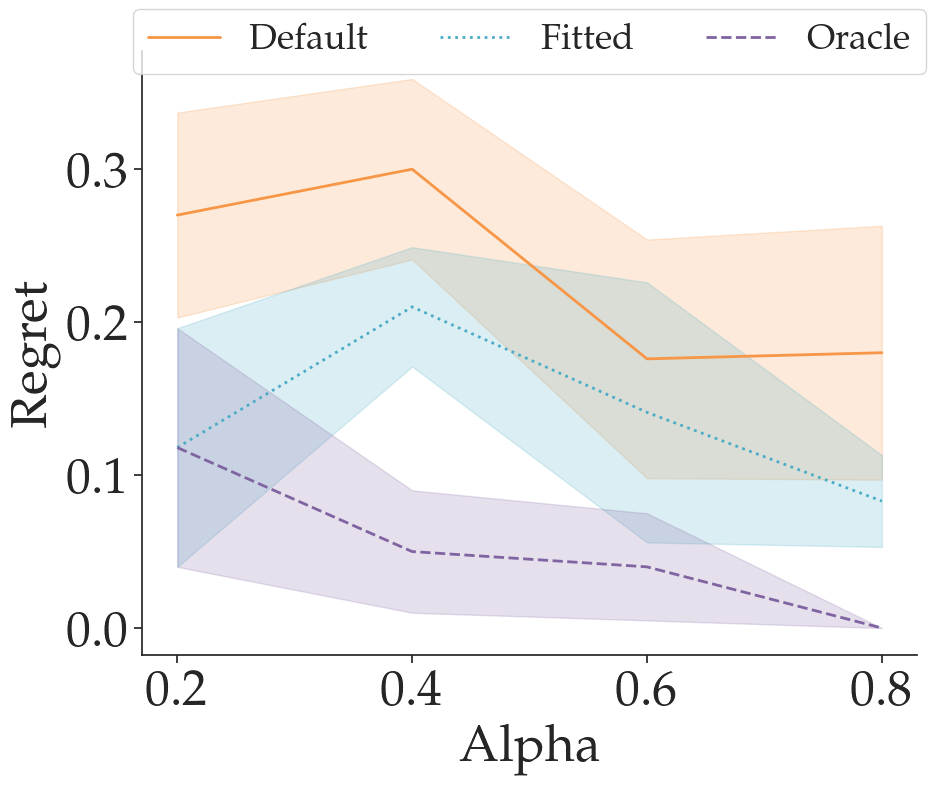}
         \caption{Extremal}
         \label{subfig:app_carextremal}
     \end{subfigure}

\caption{\textbf{Car Driving Simulator}. In (a), we show a visualization of our car-driving environment. In (b) and (c) we show the performance of modelling rationality in the car driving setting when learning from biased feedback. We find that beta-fitting results in improved reward inference in both the myopia and extremal biased settings.}
\label{fig:car}
\end{figure*}

\section{Car Driving Simulator}
In addition to our results in the gridworld navigation environment, we tested the importance of fitting $\beta$ on a car driving simulator environment we developed. In our simulator, the agent car drives along a three-lane road in the presence of other vehicles. The reward features considered are a car crash indicator, in addition to a one-hot encoded indicator of the cars lane (allowing for a lane preference in the reward function). In Figure~\ref{fig:car}, we show that learning a reward function from biased feedback can be improved by modeling rationality in the way we propose.
\section{Compute Resources Used}
The majority of experiments, as well as the user study for this paper were run on a Mid-2015 MacBook-Pro with an Intel Quad-Core i7 Processor. Certain experiments were run on a lab Intel 8-core i7 Processor.
\end{document}